\documentclass[format=acmsmall, review=false]{acmart}
\usepackage{acm-ec-25-proc}
\usepackage{booktabs} % For formal tables
\usepackage[ruled]{algorithm2e} % For algorithms
\usepackage{natbib} % For author-year citations

\SetAlFnt{\small}
\SetAlCapFnt{\small}
\SetAlCapNameFnt{\small}
\SetAlCapHSkip{0pt}
\IncMargin{-\parindent}

\setcitestyle{authoryear}

\acmYear{2025}\copyrightyear{2025}
\setcopyright{rightsretained}
\acmConference[EC '25]{The 26th ACM Conference on Economics and Computation}{July 7--10, 2025}{Stanford, CA, USA}
\acmBooktitle{The 26th ACM Conference on Economics and Computation (EC '25), July 7--10, 2025, Stanford, CA, USA}
\acmDOI{10.1145/3736252.3742570}
\acmISBN{979-8-4007-1943-1/25/07}

\usepackage{amsmath, amsthm, amsfonts}

\usepackage{graphicx}
\usepackage{bm}
\newcommand\E{\mathbb{E}}
\newcommand{\PP}{\mathbb{P}}

\newcommand\V{\text{Var}}

\newtheorem{lemma}{Lemma}
\newtheorem{prop}{Proposition}



\title[Learning treatment effects while treating those in need]{Learning treatment effects while treating those in need}

\author{Bryan Wilder}
\email{bwilder@cmu.edu}
\orcid{0000-0002-9410-020X}
\affiliation{%
  \institution{Carnegie Mellon University}
  \city{Pittsburgh}
  \state{PA}
  \country{USA}
}
\author{Pim Welle}
\email{Paul.Welle@alleghenycounty.us}
\orcid{0000-0009-8765-4321}
\affiliation{%
  \institution{Allegheny County Department of Human Services}
  \city{Pittsburgh}
  \state{PA}
  \country{USA}
}

\begin{abstract}
Many social programs attempt to allocate scarce resources to people with the greatest need. Indeed, public services increasingly use algorithmic risk assessments motivated by this goal. However, targeting the highest-need recipients often conflicts with attempting to evaluate the causal effect of the program as a whole, as the best evaluations would be obtained by randomizing the allocation. We propose a framework to design randomized allocation rules which optimally balance targeting high-need individuals with learning treatment effects, presenting policymakers with a Pareto frontier between the two goals. We give sample complexity guarantees for the policy learning problem and provide a computationally efficient strategy to implement it. We then collaborate with the human services department of Allegheny County, Pennsylvania to evaluate our methods on data from real service delivery settings. Optimized policies can substantially mitigate the tradeoff between learning and targeting. For example, it is often possible to obtain 90\% of the optimal utility in targeting high-need individuals while ensuring that the average treatment effect can be estimated with less than 2 times the samples that a randomized controlled trial would require. Mechanisms for targeting public services often focus on measuring need as accurately as possible. However, our results suggest that algorithmic systems in public services can be most impactful if they incorporate program evaluation as an explicit goal alongside targeting. 
\end{abstract}

\begin{document}

\maketitle

\section{Introduction}

A recurring challenge across many policy settings is the allocation of a limited resource under uncertainty about its benefit. Consider a policymaker who wishes to focus a limited budget for the largest overall effect. On the one hand, if the policymaker knew exactly how much each individual benefited from a given intervention, they could scale the most effective programs and target each towards the individuals who would benefit the most.  On the other hand, if there was complete uncertainty about benefits, decisions could be no better than random. Most real settings exist somewhere in the middle: policymakers believe that they can identify some individuals who are plausibly better candidates for an intervention but have little formal evidence about causal effects. Indeed, even the \textit{average} treatment effect is often not precisely known (much less heterogeneous effects), making it difficult to identify and scale the best interventions. Without precise knowledge of causal effects, resources in any given program are often allocated preferentially to individuals whose observable characteristics are thought to indicate greater need; see, e.g.\, the common use of vulnerability assessments in the allocation of housing assistance \citep{petry2021associations,shinn2022allocating} or proxy means tests in development settings \citep{grosh1995proxy,alatas2012targeting,diamond2016estimating}. It is increasingly common to measure need using predictions from machine learning models \citep{vaithianathan2020using,aiken2022machine,pan2017machine,toros2017prioritizing}. Unlike causal quantities, such models can often be estimated using existing historical data because they predict the ``baseline" risk of an adverse outcome without treatment (while the treatment effect is the difference in outcomes with and without treatment). This offers the potential to target interventions towards more vulnerable individuals but it also poses a dilemma: obtaining strongly credible evidence about treatment effects would require randomized allocations that sometimes deny treatment to individuals thought to have higher need. However, estimating even average effects would provide an often-lacking chance to identify which programs are truly effective at improving outcomes.

Here we study the question: how can policymakers optimally navigate such tradeoffs, balancing the goal of offering interventions to people who need them in the present while also gathering evidence to improve services over time? Currently, policymakers largely choose between one of two extremes. In regular practice, services are often given exclusively to those individuals deemed at highest need, an allocation rule we refer to as need-based targeting. At the opposite extreme, policymakers occasionally carve out specific settings to run a randomized controlled trial (RCT), ignoring the targeting goal entirely but enabling credible estimates of treatment effects. One reason RCTs are uncommon is that policymakers may be justifiably averse to denying a potentially beneficial intervention to high-risk individuals, even if doing so provides evidence for the future. 

We propose an optimal experimental design framework that traces out the spectrum between these extremes. In the case where a policymaker prefers either the pure randomization of an RCT or the pure needs assessment of a risk model they can still pursue those strategies, but with this framework they can operate anywhere in between. Our methods learn a class of of assignment rules that map an individual's observable features to their probability of receiving an intervention. These rules are the solution to a family of optimization problems that balance (1) the precision of the estimated average treatment effect with (2) how well allocations align with a specified measure of individual need. Our framework also allows the policymaker to easily impose additional constraints, for example on equity in allocation across different subgroups. The result is a Pareto frontier of experimental designs that optimally span the tradeoff between such competing goals.
 
We first discuss related work and then formally introduce our optimal experimental design framework. We provide a computationally lightweight approach to learn optimal assignment policies and give statistical guarantees for its finite-sample performance. Our policies can be implemented as a simple post-processing step on top of any existing means of scoring individual need. We then collaborate with the human services department of  Allegheny County, Pennsylvania to empirically evaluate our proposed methods on data from real-world service delivery settings. We analyze the tradeoff between targeting and learning in this setting, and find that, under optimized designs, the tradeoff between these two goals is often quite favorable. For examples, designs at the Pareto frontier computed by our method are often able to provide 90\% of the best possible performance in treating high-need individuals, while enabling randomization-based, gold-standard estimates of the average treatment effect at sample sizes within a factor of two of what a RCT would require. There is no way completely avoid difficult choices between targeting and learning, and the right balance to strike will inevitably be sensitive to context and the needs of particular communities. However, our results show that carefully chosen allocation rules can substantially mitigate such conflicts, offering one route towards evidence-based improvement of critical services.

\section{Related work}
Our results connect to recent interest in the difference between targeting interventions according to baseline risk (in potential outcome terms, an estimate of $Y(0)$) versus treatment effects (an estimate of $Y(1) - Y(0)$) \citep{athey2023machine,inoue2023machine}. Some have argued that targeting by risk is difficult to rationalize in welfare terms because high-risk individuals may not realize the greatest benefit \citep{haushofer2022targeting}, while others have illustrated settings where targeting based on risk selects a population with a higher ``ceiling" for potential impact \citep{heller2022machine}. Our goal in this paper is \textit{not} to take a stance about what preferences policymakers ought to have with respect to this question. There are a variety of reasons policymakers may prefer either strategy, especially in the setting we consider where high-quality estimates of causal effects are not available at the outset. In some domains, beliefs about which individuals are likely to benefit from a service may already be baked into eligibility criteria or risk assessments. In others, policymakers may see intrinsic value in offering assistance to individuals with high need, even if they do not benefit as much as others. Our framework is agnostic as to the content of these preferences; we take them as encapsulated in a fixed \textit{targeting utility function} and our goal is simply to enable treatment effect estimation while sacrificing as little of this pre-specified utility as possible. While one empirical motivation for our work is the widespread adoption of algorithmic risk assessments, our methods can be equally well applied to preferences informed by treatment effect estimates if these are available. 

We can draw a similar distinction between our goals and those of the online learning paradigm. For example, one line of work models the assignment of individuals to interventions as a contextual bandit \citep{agarwal2014taming,foster2020beyond,dimakopoulou2019balanced}. However, such methods have seen relatively few uses in real-world policy settings. The motivation for our problem formulation is that policymakers may have multiple incommensurable goals, in which case the single long-term regret objective optimized in online learning may fail to capture their preferences. Such objectives can unfold on several levels. First, as discussed above, offering assistance to individuals with high need is often considered worthwhile in itself, even if those individuals do not benefit as much from treatment (and hence contribute less to the standard objective for a contextual bandit). Second, treatment effect estimates inform multiple kinds of decisions: both who to offer a given intervention to, but also which programs to offer in the first place. Indeed, identifying and scaling the best-performing interventions (via an estimate of the ATE) may often be more impactful than improving the targeting of individual interventions \citep{perdomo2024relative,liu2023reimagining}. Finally, policymakers may be unwilling to tolerate the significant amount of randomized exploration that online learning methods often require. For the small sample sizes common in social service settings, such exploration may impose costs qualitatively similar to a RCT. For all of these reasons, we formulate an explicitly multi-objective problem which presents the policymaker with a Pareto frontier between targeting individuals in need and estimating treatment effects (along with any other goals such as fairness). Perhaps closest to our work is that of \citet{henderson2023integrating}, who propose a bandit framework that integrates both exploitation and statistical estimation as goals. However, their method still requires online exploration and sequential updates. In comparison, our framework is best suited for the case where policymakers wish to learn about treatment effects while \textit{always} ensuring a specific level of performance in targeting high-need individuals, instead of minimizing asymptotic regret as in online learning. This presentation of a Pareto curve between multiple objectives can be seen in the spirit of \citet{rolf2020balancing}'s proposal for machine learning to explicitly balance objectives such as profit, welfare, and fairness. Operationally, our method is also completely non-adaptive and easily implemented as a post-processing step to any existing risk model. This helps circumvent common challenges in implementing online learning such as delayed outcomes \citep{joulani2013online} and inability to make frequent updates \citep{perchet2016batched}.

Finally, our work is related to other efforts to design experiments that satisfy other desirable properties. For example, \citet{narita2021incorporating} and \citet{chassang2012selective} develop frameworks which offer individuals customized allocation probabilities in order to allow incentive-compatible elicitation of their private value for being in the treatment arm. However, they do not explicitly optimize the tradeoff between utility and statistical power as in our framework. \citet{owen2020optimizing} study exploration-exploitation tradeoffs in tie-breaker designs which uniformly randomize individuals within some distance of the decision boundary for an intervention. However, they do not consider the design of optimal assignment probabilities as a function of individuals' covariates, or provide the ability to impose constraints such as fairness over such policies.  

% Our approach can also be seen as part of a line of work which explicitly frames experimental design as an optimization problem \citep{bertsimas2015power,kasy2016experimenters,banerjee2020theory,kallus2018optimal,morgan2012rerandomization}. For example, \citet{harshaw2023balancing} seek designs with greatest balance subject to a constraint on worst-case error in treatment effect estimates. However, the motivation for our formulation is distinct from that in previous work: we model a policymaker who seeks to control the costliness of an experiment in terms of its deviation from treating higher-need individuals. 

\section{Methods}

Consider the problem of a policymaker who would like to determine a scheme for allocating a limited quantity of a treatment or intervention, while balancing multiple different goals. Specifically, they would like to both reach individuals who are greatest need according to some pre-specified metric while simultaneously learning about the average treatment effect that the intervention offers. The policymaker chooses a function $p(X): \mathcal{X} \to [0,1]$ which gives the probability of assigning each individual treatment based on their features $X \in \mathcal{X}$. Treatment is assigned to each individual independently with this probability, after which their corresponding outcomes are observed. Finally, the decision maker estimates the average treatment effect on some target population.  Let $Y(0), Y(1)$ denote potential outcomes absent and with treatment, respectively. We assume that $X, Y(0), Y(1) \sim \PP$ iid for some joint distribution $\PP$. Let $\tau = \E[Y(1) - Y(0)]$ denote the ATE and $\tau(X) = \E[Y(1) - Y(0)|X]$ denote the conditional average treatment effect (CATE). We start by formalizing objective functions that quantify how well a policy enables estimation of the ATE. Then, we introduce our constrained optimization framework to trade off estimation efficiency with other social desiderata and finally provide efficient algorithms and finite-sample guarantees to solve the resulting policy optimization problem.   

\subsection{Formalization of objectives for causal estimation}
We start by assuming that the decision maker wishes to estimate the ATE with respect to the distribution $\PP$, and later consider the case where they wish to estimate the ATE on specific subpopulation. Let $u(X)$ denote a \textit{targeting utility} function chosen by the decision maker which maps an individual's observed covariates to the utility of offering treatment to that individual. In our running example, $u$ may be an estimate of the probability of an adverse outcome absent treatment (i.e., $u(X) = \widehat{E}[Y(0)|X]$), but our methods are agnostic as to how $u$ is chosen. We assume that $u$ is normalized so that $u(X) \in [0,1]$ with probability 1. 

To formalize our objective, we start by recalling the efficiency bound for the variance of the ATE estimate \citep{hahn1998role,imbens2009recent}, given by 
\begin{align*}
   \mathbb{V}_{\text{ATE}} = \E\left[\frac{\V(Y(1)|X)}{p(X)} + \frac{\V(Y(0)|X)}{1- p(X)}  + (\tau(X) - \tau)^2\right].
\end{align*}
The efficiency bound quantifies the lower possible variance for estimating the ATE. Since in experimental settings we will be able to use unbiased estimators, this is also equivalent to the best possible mean-squared error. By focusing on the efficiency bound, our formulation is agnostic to the exact choice of estimator. Since under appropriate additional regularity conditions, various estimators achieve the efficiency bound (e.g., the augmented inverse-propensity weighted or doubly-robust estimator), our results can also be seen as targeting the variance achieved by a properly chosen estimator. 

We seek assignment probabilities $p(X)$ that minimize this estimation error subject to a constraint on utility. However, the outcome variance terms in the numerator typically cannot be estimated without observations of the potential outcomes, i.e., before running the trial.  This dilemma is closely related to the question of the \textit{Neyman variance} \citep{neyman1992two} in experimental design: RCTs that randomize equally between treatment and control groups are optimal (have smallest variance) when the two outcomes have equal variance, while if the variance is higher in one arm it is optimal to allocate a greater fraction of subjects to that arm. There is a recent line of work that attempts to adapt assignment probabilities over the course of the RCT based on intermediate estimates of the variances \citep{dai2024clip,zhao2023adaptive}. However, in this work we focus on developing fixed, non-adaptive designs due to their much greater implementability in service delivery settings (see discussion in related work). 

In common practice, most RCTs simply use complete randomization, assigning with probability 0.5 to both treatment and control. This can be seen as optimizing for the case where the variances are equal, $v_0(X) = v_1(X) = C \,\,\,\forall X$. In our framework, we propose to accommodate different kinds of prior information that may be available about $v_0$ and $v_1$ by selecting an uncertainty set $\Sigma$ for the variance functions and optimizing relative to the worst-case scenario within that set. That is, let
\begin{align*}
    \mathbb{V}_{\text{ATE}}(v_0, v_1, p) = \E\left[\frac{v_1(X)}{p(X)} + \frac{v_0(X)}{1- p(X)}  + (\tau(X) - \tau)^2\right]
\end{align*}
denote the efficiency bound as a function of the conditional variance functions $v_1$ and $v_0$. Then, we seek an assignment policy within some constrained set $\mathcal{P}$ which solves the minmax problem
\begin{align*}
    \min_{p \in \mathcal{P}} \max_{v_0, v_1 \in \Sigma}  \mathbb{V}_{\text{ATE}}(v_0, v_1, p).
\end{align*}
The properties of the resulting solution will clearly be determined by the choice $\Sigma$. We emphasize though that the resulting policies will permit unbiased estimates of treatment effects for \textit{any} choice of $\Sigma$ that the analyst makes, regardless of how misspecified it turns out to be. Imposing correctly specified structure on $\Sigma$ may simply allow more efficient estimation, analogous to the specification of a working variance model in statistical tasks. We next propose possible instantiations for $\Sigma$ depending on the amount of a-prior knowledge available about the variances. 
\paragraph{Example 1} In the case where the analyst is completely agnostic as to the structure of the variances, we propose to optimize under the assumption only that the variances are bounded, i.e., there is some constant $C$ such that $v_0(X) \leq C$ and $v_1(X) \leq C$ holds for all $X$. This corresponds to the set $\Sigma_{\infty, C} = \{v_0, v_1: ||v_0||_\infty \leq C, ||v_1||_\infty \leq C\}$. However, the worst-case scenario within this set is to set both $v_1$ and $v_0$ to the maximum possible value for every $X$. As a simple consequence, we obtain that 
\begin{prop}
    Let $p^*$ be an optimal solution of the problem $\min_{p \in \mathcal{P}} \E\left[\frac{1}{p(X)} + \frac{1}{1 - p(X)}\right]$. Then, $p^*$ is also an optimal solution to the problem  $\min_{p \in \mathcal{P}} \max_{v_0, v_1 \in \Sigma_{\infty, C}}  \mathbb{V}_{\text{ATE}}(v_0, v_1, p)$ for any $C \geq 0$. \label{prop:equal-minmax}
\end{prop}
Effectively, Proposition \ref{prop:equal-minmax} says that if we do not impose any further assumptions on the variance, the minmax solution is to treat the variances in the two groups as equal. If $\mathcal{P}$ were entirely unconstrained, the minmax policy would simply recover the standard RCT where $p(X) = 0.5$ for all $X$; the focus of our framework in this case will be on how constraints on the expected targeting utility of the policy cause deviations from the typical equal-probability design. We suggest this uncertainty set as a good, simple default unless strong prior knowledge about heteroskedasticity is available. We will also show later that adopting this uncertainty set empirically leads to performance almost as strong as optimizing with complete knowledge of the variance structure.  

\paragraph{Example 2} In some cases, administrative data is available from before a new intervention is deployed. In these cases, we may be able to estimate $\V[Y(0)|X]$ as some $\hat{v}_0(X)$ while $\V[Y(1)|X]$ is unknown. A natural uncertainty set is to constrain how far $\V[Y(1)|X]$ can deviate from $\hat{v}_0(X)$. We will work with flexible specifications where assume that $\V[Y(1)|X] \leq a(X) \cdot \hat{v}_0(X)$ for some function $a(X) \geq 0$, leading to the uncertainty set 
\begin{align*}
    \Sigma_a = \{\hat{v}_0, v_1: v_1(X) \leq a(X) \cdot \hat{v}_0(X)\}.
\end{align*} 
A more conservative variation where $v_0$ is constrained to lie below an upper confidence bound for $\V[Y(0)|X]$ instead of being set at the point estimate is also possible and results in similar structure. The question is how to choose the function $a(X)$. Without any further knowledge, $a$ might simply be set to a constant reflecting the desired degree of robustness to large variance in the treatment arm. However, when outcomes are binary (as for all of the examples in our motivating application of human service delivery), more informed choices are possible. For binary variables, we are guaranteed that $\V[Y(1)|X] \leq \frac{1}{4}$ even in the worst case where $\Pr(Y(1)|X) = 0.5$, so setting $a(X) = \frac{1}{4\hat{v}_0(X)}$ guarantees that $\Sigma_a$ will contain the true heteroskedasticity structure. Under some circumstances, we may even be able to sharpen the bounds further. For example, suppose that we are willing to optimize under the assumption that the intervention never makes outcomes worse in expectation, i.e., $\tau(X) \leq 0$ for all $X$. Then, for any $X$ where $\Pr(Y(0) = 1|X) \leq 0.5$, we are guaranteed that $\Pr\left(Y(1) = 1|X\right) \leq \Pr\left(Y(0) = 1|X\right)$ and so $\V[Y(1)|X] \leq \V[Y(0)|X]$ and we can set $a(X) = 1$. This case occurs frequently in our application domains because many outcomes of interest (e.g., mortality, homelessness, or jail entry) occur less than half the time. 

Regardless of the choice of $a$, we find that the minmax problem can be simplified in a desirable fashion reminiscent of the first model studied:
\begin{prop}
    Let $p^*$ be an optimal solution of the problem $\min_{p \in \mathcal{P}} \E\left[\frac{a(X) \hat{v}_0(X)}{p(X)} + \frac{\hat{v}_0(X)}{1 - p(X)}\right]$. Then, $p^*$ is also an optimal solution to the problem  $\min_{p \in \mathcal{P}} \max_{v_0, v_1 \in \Sigma_{a}}  \mathbb{V}_{\text{ATE}}(v_0, v_1, p)$ for any choice of $a : \mathcal{X} \to R^+$. \label{prop:estimated-minmax}
\end{prop}
This structure arises because the worst case scenario again places both variance terms at their pointwise upper bounds. 

\subsection{Constrained optimization problem for assignment policies}
We study a family of policy optimization problems that subsumes both of the above models. The outcome variance model is accounted for by fixing two functions $a_0, a_1: \mathcal{X} \to R^+$ that will weight the control and treatment terms in the efficiency bound, respectively. E.g., Example 1 is recovered by setting $a_0(X) = a_1(X) = 1$ and Example 2 by setting $a_0(X) = \hat{v}_0(X)$ and $a_1 = a(X)\hat{v}_0(X)$. We aim to find allocation probabilities which solve an optimization problem of the form:
\begin{align}
    &\min_p \E\left[\frac{a_1(X)}{p(X)} + \frac{a_0(X)}{1-p(X)}\right] \nonumber \\
    &\E[p(X)u(X)] \geq c \label{problem:main}\\
    &\E[p(X)] \leq b \nonumber\\
    &p(X) \in [\gamma,1-\gamma] \quad \forall X \in \mathcal{X}. \nonumber
\end{align}
The objective is to minimize the worst-case variance of the ATE estimate relative to the chosen uncertainty set (as instantiated in $a_0$ and $a_1$). The first constraint imposes that the assignment rule has expected utility at least $c$ with respect to $u$. The second constraint enforces budget feasibility, that at most a fraction $b$ of individuals are offered treatment. Finally, we restrict the assignment probabilities to the interval $ [\gamma,1-\gamma]$ for a small constant $\gamma$ chosen by the user since the objective becomes undefined at $p(X) = 0$ or 1. In practice, we find that the constraints involving $\gamma$ are not typically binding because the objective function rapidly increases as $p(X)$ approaches 0 or 1, rendering boundary solutions suboptimal.

This core formulation can also be extended with additional user-specified constraints. For example, we may wish to impose a constraint on equity in resource allocation across different population groups. Suppose that we specify two subgroups $\mathcal{G}_0, \mathcal{G}_1 \subseteq X$. Natural constraints could be that the groups have similar expected utility, 
\begin{align}
    |\E[p(X)u(X)|X \in \mathcal{G}_0] - \E[p(X)u(X)|X \in \mathcal{G}_1]| \leq \epsilon \label{eq:fairness-constraint}
\end{align}
or are offered similar treatment probabilities, 
\begin{align}
    |\E[p(X)|X \in \mathcal{G}_0] - \E[p(X)|X \in \mathcal{G}_1]| \leq \epsilon. 
\end{align}
In general, our technical framework can accommodate any such constraints so long as the resulting feasible set is convex in $p$ (as is the case for these two examples). In what follows, we let the optimization problem have constraints of the form
\begin{align*}
    \E[g_j(p(X_i), X_i)] \leq c_j \quad j = 1...J
\end{align*}
where $J$ is the total number of constraints. We assume that each function is normalized so that $g_j(p(X_i), X_i) \in [0, 1]$ with probability 1 and $g_j(p(X_i), X_i)$ is 1-Lipschitz with respect to $p(X_i)$. Both conditions are easily enforced for all of the example constraints above.

% The assignment probabilities must also respect a set of user-specified constraints $\{g_j\}_{j = 1}^J$.  We focus on two particular examples of such constraints:

% First, a constraint on the estimated \textit{recall} (or sensitivity) with which the policy allocates interventions to individuals with $Z = 1$ (i.e., an adverse outcome absent treatment). In order to guarantee sensitivity at least $c_j$, we require that $\E[p(X)|Z = 1] \geq c_j$, or equivalently $\E[p(X)E[Z|X]] \geq \frac{c_j}{\E[Z]}$. Using $\mu(X)$ as a plugin estimate of $E[Z|X]$, we obtain the final constraint
% \begin{align*}
%     \E[p(X)\mu(X)] \geq \frac{c_j}{\E[Z]}.
% \end{align*}

% Second, a constraint on equity in resource allocation. While there are many possible ways of formalizing equity appropriate to different settings, for concreteness we focus on a constraint in \textit{recall parity}, sometimes referred to as equality of opportunity. Roughly, this requires for two specified subgroups $\mathcal{G}_0, \mathcal{G}_1 \subseteq X$, the probability of receiving an intervention conditional on $Z = 1$ is similar between the groups. Formally, this is 
% \begin{align*}
%     |\E[p(X)|X \in \mathcal{G}_0, Z = 1] - \E[p(X)|X \in \mathcal{G}_0, Z = 1]| \leq \epsilon
% \end{align*}
% which can then be rewritten as two linear inequality constraints using the plugin estimator $\mu(X)$ similarly to above. 

There are two main difficulties in solving the above optimization problem. First, it is over a functional decision variable $p$ (a mapping from covariates to assignment probabilities). Second, the objective and constraints are in terms of expectations over the data generating distribution, while in practice we will only have access to finite samples. We next introduce efficient algorithms to compute near-optimal policies with finite-sample guarantees.

\subsection{Computing optimal policies}
Our aim is to find an assignment probability function $p$ that solves the above optimization problem, given a sample  $X_1...X_n \sim \PP$. We emphasize that our approach only requires unlabeled data to inform the covariate distribution. The high-level idea is to use this sample to estimate the optimal value of the dual parameter for each constraint. As each individual arrives, we can then solve a \textit{separate} optimization problem to compute their assignment probability, balancing between the need to adhere to the constraints (as represented by the dual variables) versus the variance that this individual would contribute to the ATE estimate. For example, an individual for whom $u(X)$ is large would have a larger dual term encouraging assignment to treatment (in order to boost utility), which can be weighed against the variance contributed by large assignment probabilities. To formalize this idea, we take the dual of the population optimization problem to obtain
\begin{align*}
    \max_{\lambda \geq 0} \min_{p \in [\gamma,1-\gamma]^\mathcal{X}} \E\left[\frac{a_1(X)}{p(X)} + \frac{a_0(X)}{1-p(X)} + \sum_{j = 1}^J \lambda_j(g( p(X), X) - c_j)\right] 
\end{align*}
where $\lambda_j$ is the dual variable associated constraint $j$. The dual has the attractive property that the inner objective is separable across $X$, allowing us to push the min inside the expectation:
\begin{align*}
    \max_{\lambda \geq 0}  \E\left[\min_{p(X) \in [\gamma,1-\gamma]} \frac{a_1(X)}{p(X)} + \frac{a_0(X)}{1-p(X)}  + \sum_{j = 1}^J \lambda_j(g(p(X), X) - c_j)\right].
\end{align*}
Our strategy will be to produce estimates $\widehat{\lambda}$ of the the optimal dual parameters  on the training set by solving the \textit{sample} problem
\begin{align}
    \max_{\lambda \geq 0}  \frac{1}{n}\sum_{i = 1}^n \left[\min_{p(X_i) \in [\gamma,1-\gamma]} \frac{a_1(X_i)}{p(X_i)} + \frac{a_0(X_i)}{1-p(X_i)}  + \sum_{j = 1}^J \lambda_j(g(p(X_i), X_i) - c_j)\right]. \label{problem-sample}
\end{align}
This is a strongly convex optimization problem in $n$ variables that can easily be solved using standard methods. Then, at test time, we compute $p(X)$ for each incoming individual $X$ by solving the inner minimization problem at the optimal dual parameters, i.e., we compute 
\begin{align}
    \widehat{p}(X) = \text{argmin}_{p(X) \in [\gamma,1-\gamma]} \frac{a_1(X)}{p(X)} + \frac{a_0(X)}{1-p(X)}  + \sum_{j = 1}^J \widehat{\lambda}_j(g(p(X), X) - c_j) \label{eq-optimal-p}
\end{align}
separately for each individual and randomize them to be treated with probability $\widehat{p}(X)$. Effectively, the dual parameters tell the test-time algorithm how much to weight the constraints (e.g., utility and budget considerations) compared to variance.  One potentially desirable property of this approach is that it does not require any joint computations over the entire cohort of individuals who are candidates for treatment; the allocations can be computed and sampled entirely separately. This may be  necessary in settings where individuals arrive in a rolling fashion, as in e.g.\ most operational social or health services.

\subsection{Finite-sample guarantees}

We now turn to establishing the number of samples of $X$ that are required to ensure that we obtain a policy which is both close to satisfying the constraints, and near-optimal in terms of variance, with high probability.  The main idea is that optimal policies for this problem are parameterized just by the values of the dual variables, so we will obtain a near-optimal policy if the value of the Lagrangian is approximated well across the entire set of possible values that the duals could take. We first present a generic sample complexity bound for any set of convex constraints which quantifies the dependence on the duals through two quantities: a high-probability bound on the maximum value that the duals can take, and the minimum value of the variance proxies $a_0$ and $a_1$ (which influences the smoothness of the relationship between the duals and the resulting primal solution). We will then instantiate this result for linear constraints structures and interpret and/or further bound these quantities for constraints like the ones proposed above. We start by stating the generic result, which gives conditions under which our approach produces policies within $\epsilon$ of optimality and feasibility.

% Let $\kappa = \sup_{X_1, X_2 \in \mathcal{X}}\mu(X_1) - \mu(X_2)$.  $\kappa$ plays the role of quantifying how well-conditioned the constraints are: if all of the values of $\mu$ are very close together, the constraints are nearly colinear and it becomes harder to control the values of the dual variables. 

% We can now state the main sample complexity result:
\begin{prop}
    Given $n$ iid samples of $X$ from $\PP$, let the solution to the sample dual optimization Problem \ref{problem-sample} be $\widehat{\lambda}$. Let $\widehat{p}$ be the policy obtained by solving the associated Lagrangian for any given $X$ and $p^*$ denote the optimal solution to the population problem. Suppose that for some $d > 0$, $||\widehat{\lambda}||_\infty \leq d$ with probability at least $1 - \delta_1$. Suppose also that there exist constants $a_{\text{min}}, a_{\text{max}} \geq 0$ such that $a_{\text{min}}\leq a_0(X), a_1(X) \leq a_{\text{max}}$ for all $X$. Fix any $\epsilon, \delta_2 > 0$. In order to guarantee that with probability at least $1 - \delta_1 - \delta_2$, 
    \begin{align*}
        &\E\left[\frac{a_1(X)}{\widehat{p}(X)} + \frac{a_0(X)}{1-  \widehat{p}(X)}\right] \leq \E\left[\frac{a_1(X)}{p^*(X)} + \frac{a_0(X)}{1-  p^*(X)}\right] + \epsilon\\
        &\E[g_j(\widehat{p}(X), X)] \leq c_j + \epsilon \quad j = 1...J
        % &\E[\widehat{p}(X) \mu(X)] \geq c -  \epsilon\\
        % &\E[\widehat{p}(X)] \leq b + \epsilon,
    \end{align*} 
    it suffices to have $n = O\left(\frac{J^3d^2 }{\epsilon^2 a_{\text{min}}^2}\log\frac{Jd}{\epsilon\delta_2} \right)$ samples. \label{prop-sample-complexity}
\end{prop}
Rearranging the theorem statement, the loss in both the objective and the constraints shrinks with high probability at a rate of $O(n^{-\frac{1}{2}})$ when $J, d$, and $a_{\text{min}}$ are fixed. A proof is given the appendix. The main idea is that optimal policies for this problem are parameterized just by the values of the dual parameters. Notably, this allows us to escape any dependence on ``dimension" terms for the policy class except the number of constraints $J$. By contrast, sample complexity guarantees for policy learning from observational data typically require imposing a bound on a complexity measure such as the VC dimension or covering number \citep{athey2021policy,chernozhukov2019semi,swaminathan2015batch}, regardless of whether the optimal policy actually belongs to a small-dimensional class. In our setting, we obtain convergence to the optimum at the root-$n$ parametric rate without the need for a separate realizability assumption.  

The convergence rate does depend on $||\widehat{\lambda}||_\infty$, the maximum value of the optimal dual parameters for the sample problem. Intuitively, the magnitude of the dual variables quantifies how sensitive the objective function is to solutions that leave some slack in the constraints due to sampling noise, e.g., slightly underutilize the budget. As discussed below, for well-behaved constraint structures, this term can be controlled using the fact that the objective has bounded gradients. It also depends on $a_{\text{min}}$, a lower bound on the variance proxies chosen by the analyst. In isolation, a smaller value of $a_{\text{min}}$ implies that more samples are required in Proposition \ref{prop-sample-complexity} because changes to $\widehat{\lambda}$ can have a larger impact on $\hat{p}$. In our analysis below, after we account below for the impact that the variance proxies have on $||\widehat{\lambda}||_\infty$, what will end up mattering is the ratio $\frac{a_{\text{max}}}{a_{\text{min}}}$. Intuitively, if the analyst mandates greater heteroskedasticity across values of $X$, the objective function becomes sensitive to a smaller set of $X$'s and we need more samples to ensure this region is covered well. For our suggested uncertainty sets, $\frac{a_{\text{max}}}{a_{\text{min}}}$ can typically be regarded as a constant (e.g.\ $\frac{a_{\text{max}}}{a_{\text{min}}} = 1$ in Example 1).   

In order to illustrate conditions under which sample complexity can be further bounded in this manner, we turn to the case where the constraints are linear in $p$ (as are all the constraints discussed above), represented as $\E[g_j(X) p(X)] \leq c_j$ for functions $g_j$ that depend only on $X$. Let $g(X_i) = [g_1(X_i)...g_J(X_i)]$ collect the coefficients for all of the constraints with respect to individual $i$. Note that $g$ is a random variable which depends on the iid draws of the $X$. We will show that the key terms appearing in Proposition \ref{prop-sample-complexity} can be bounded whenever $g$ satisfies a standard small ball condition from random matrix theory \citep{lecue2017sparse,yaskov2016controlling}: 
\begin{align*}
    \inf_{||\nu||_2 = 1} \Pr\left(|g(X)^T\nu| \geq \alpha\right) \geq \beta
\end{align*}
for some $\alpha, \beta > 0$. Roughly, this says that the entries are $g$ are not too small with very high probability; e.g., in Problem \ref{problem:main} it requires that the utility function is not too concentrated on a low-probability set of individuals. Under this condition, the KKT system defining $\widehat{\lambda}$ is well-conditioned with high probability and we obtain
\begin{prop}
Suppose that the constraints are linear in $p$ and that $g(X_i)$ obeys the small ball condition with parameters $\alpha, \beta > 0$. Then $||\widehat{\lambda}||_\infty \leq \frac{\sqrt{2}a_{\text{max}}}{\alpha \sqrt{\beta} \gamma^2}$ with probability $1 - \Theta(\exp(-\beta^2 n))$ and so $n = O(\frac{J^3}{\epsilon^2} \frac{1}{\alpha^2 \beta \gamma^4} \left(\frac{a_{\text{max}}}{a_{\text{min}}}\right)^2 \log \frac{J}{\alpha\beta\gamma\delta_2 \epsilon} )$ samples suffice for the guarantees of Proposition \ref{prop-sample-complexity} to hold with probability $1 - \Theta(\exp(-\beta^2 n)) - \delta_2$. \label{corollary-small-ball}
\end{prop}
To interpret this result, consider the case of Problem \ref{problem:main} where we have only the utility and budget constraints. Here, we have $g(X) = [u(X), 1]$ since every individual counts equally towards the budget constraints. Since $u(X) \in [0,1]$, the minimizing $\nu$ in the small ball condition will be $\nu = [1, 0]$ and the condition reduces to a question about the distribution of the variable $u(X)$: it will be satisfied with reasonable constants as long as it is not the case that the utility is concentrated entirely on a very uncommon set of individuals. Conversely, if the utility function indeed assigns most of its weight to a very small-probability subset, then we will need more samples to ensure that region of the covariate space is well-represented. Even under a worst-case distribution of $u$, this cannot happen as long as the average utility value is not too small. Formally, a reverse Markov inequality yields that:
\begin{prop}
    In Problem \ref{problem:main}, the small-ball condition is satisfied for any choice of $\alpha < \E[u(X)]$ and $\beta = \frac{\E[u(X)] - \alpha}{1 - \alpha}$. \label{prop:reverse-markov}
\end{prop}
In our running example where $u$ is given by a risk score, $\E[u(X)]$ is just the frequency of the outcome of interest and can be treated as a constant. In practice, we find that that the distributions of predicted risk in our empirical setting easily satisfy a small-ball condition; see the empirical CDFs of $u(X)$ for each dataset in Figure \ref{fig:utility-cdf}. For example, we have that $\Pr\left(u(X) \geq 0.3\right) \geq 0.69$ for the first dataset (reentry) and $\Pr\left(u(X) \geq 0.3\right) \geq .84$ for the second dataset (housing).

\begin{figure}
    \centering
    \includegraphics[width=0.5\linewidth]{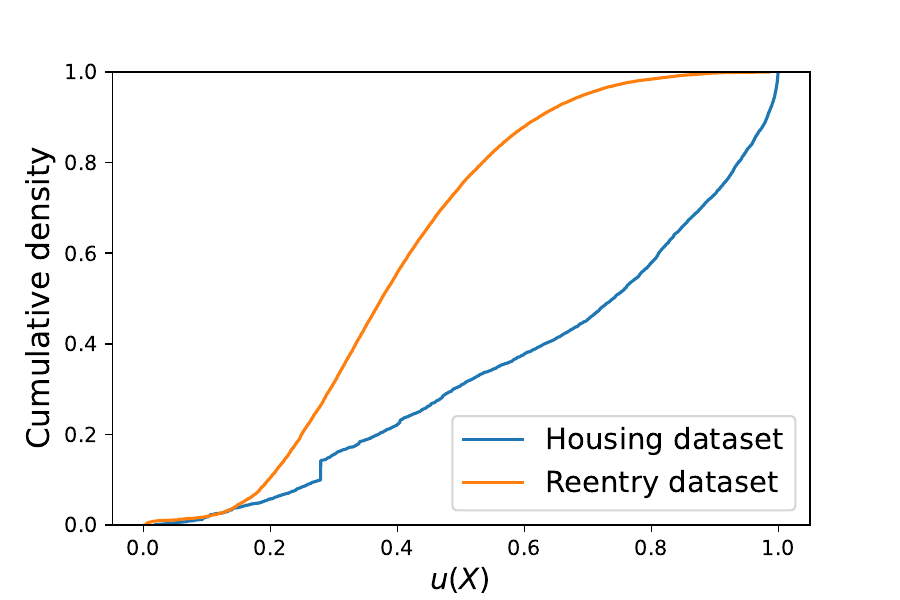}
    \caption{Empirical CDF of the score $u(X)$ for each dataset. }
    \label{fig:utility-cdf}
\end{figure}

\subsection{Alternate estimands}
\label{section:alternate-estimands}
In some settings, policymakers may wish to identify effects other than the ATE. For example, they may wish to identify the average treatment effect only for the $\alpha$-fraction of the population greatest risk for an adverse outcome (the beneficiaries under need-based targeting), or some other subgroup of particular interest. Our framework can be naturally modified to accommodate other estimands. In particular, let $\mathcal{S} \subseteq \mathcal{X}$ be the target population for which we wish to estimate the group average treatment effect. The objective function becomes 
\begin{align*}
\E\left[\left(\frac{a_1(X)}{p(X)} + \frac{a_0(X)}{1-p(X)}\right)\Big|X \in \mathcal{S}\right] = \frac{1}{\Pr[X \in \mathcal{S}]} \cdot \E\left[1\left[X \in \mathcal{S}\right]\left(\frac{a_1(X)}{p(X)} + \frac{a_0(X)}{1-p(X)}\right)\right]
\end{align*}
which counts individuals towards the variance only if they belong to the target population. The modified optimization problem can then be solved identically to as discussed above.

\subsection{Handling uncertainty in $u$}
In many settings policymakers may face uncertainty in knowing the utility $u$ that they truly wish to target on. For example, when $u$ is intended to target higher-risk individuals, any learned risk model may only be a noisy proxy for the ideal of directly allocating resources to individuals will experience adverse outcomes in the future. Thus far, our methods have been deliberately agnostic about where $u$ comes from and whether there is uncertainty about its relationship to a true target of interest. Different applications will involve different targets for prediction, classes of models, and challenges in estimation, all of which influence uncertainty quantification. Sometimes $u$ may not be a machine learning model at all: hand-constructed scoring rules are common in settings like public housing, social services, often aiming to capture a sense of need instead of making a prediction. Our general approach thus takes $u(X)$ at "face value" as the policymaker's priorities.

However, we briefly sketch a strategy to account for uncertainty when $u(X)$ is a predictive model for some outcome $Z \in \{0,1\}$ (i.e., where $u$ is a risk score, as in many applications cited in above). Uncertainty arises because individuals may have $Z = 1$ (a bad outcome) without having large $u(X)$ if the risk model is imperfect. Suppose want to guarantee that a fraction $\eta$ of individuals who would have $Z = 1$ are offered treatment but only $X$ is observed at treatment assignment. If $u(X)$ were the Bayes-optimal predictor for $Z$, setting the constraint $E[p(X)u(X)] \geq \eta E[u(X)]$ would suffice. However if $u$ is not well-calibrated within some strata of $X$, this might not guarantee coverage with respect to the actual outcome. Intuitively, we can simply shade the value of the constraint higher to compensate. Formally, suppose we have a held-out dataset $\{(X_i, Z_i)\}_{i = 1}^n$. We can search for a value $d$ such that, when the constraint $E[p(X)u(X)] \geq d$ is imposed, $\frac{1}{n}\sum_{i = 1}^n Z_ip(X) = \eta\cdot \frac{1}{n}\sum_{i = 1}^n Z_i$. If such a $d$ exists, standard concentration bounds allow us to translate this into a guarantee that  $E[Z \cdot p(X)] \geq  \eta E[Z] - \epsilon$ with high probability.

\section{Results}

\subsection{Experimental setup}
\paragraph{Setting}
We apply our methods to the design of evaluations for interventions in a human services setting. We use data provided by the Allegheny County Human Services Department, which administers programs for mental and behavioral health, substance use treatment, housing assistance, support for formerly incarcerated people, aging, and more. The department serves a population of over 1.2 million people and allocates a total budget of more than \$1 billion across such programs. It regularly uses predictive risk models to prioritize individuals with greater need for access to program with limited resources, but also faces the overall challenge of determining which programs are working in order to concentrate resources on more effective programs. 

\paragraph{Datasets} We use two datasets provided by the Allegheny County Human Services Department. The first covers 9213 individuals who were eligible in the past for public housing assistance (prior to the creation of a new housing program). We use 40\% of the data to train a random forest classifier to predict the probability of having at least 4 emergency department visits in the next year. The base rate of this outcome is approximately 54\%. The covariates $X$ are a set of 83 features collected by the county from administrative data sources, including age, education level, previous appearances in the court system, and past utilization of behavioral health, housing, or medicaid services, mimicking the process used in practice at the county to construct predictive models in variety of service areas. We then use the remaining 60\% of the data to estimate the performance of the optimal policy. In line with current practice, the utility function $u(X)$ is the risk score output by the classifier, indicating a preference for treating higher-acuity individuals. 

The second dataset contains the outcomes of 76052 incarcerated people released from jail. The primary outcome of interest is reentry to jail within the next year; the frequency of this outcome in our data is approximately 39\%. We predict status-quo reentry risk using a random forest classifier and the same set of administrative data features from the first domain, and set the utility function $u$ to be the predicted reentry risk. We use 60\% of the data to train the predictive model and the remaining 40\% to estimate the performance of the optimal policy.  

\paragraph{Simulation setup} In order to simulate the performance of potential policies on the test set, we need to fix a distribution of the outcome variables under control and treatment, which in turn implies the values of the outcomes variances that appear in the efficiency bound. We set $\Pr(Y(0) = 1|X)$ to be equal to the predicted probability from the random forest classifier that was fit on the train set. Since the outcome is binary, the variance is $\V(Y(0)|X) = \Pr(Y(0) = 1|X) \cdot (1 - \Pr(Y(0) = 1|X))$. We then add a synthetic treatment effect to arrive at the distribution of $Y(1)$. In our primary analysis, we simulate the treatment effect $\tau(X) = -\beta \cdot \Pr(Y(0) = 1|X)$, where $\beta$ is the relative reduction in the frequency of adverse outcomes. Accordingly, we have $\Pr(Y(1)|X) = (1-\beta)\Pr(Y(0) = 1|X)$ In absolute terms, this generates a heterogeneous effect where individuals with higher baseline risk experience a larger effect (as service designers typically hope). Later, we conduct sensitivity analyses to the case where treatment effects are either constant or U-shaped (i.e., with largest effects for individuals of moderate risk). Our main results are shown for $\beta = 0.1$, with other values of $\beta$ and alternate forms of treatment effect heterogeneity shown in the appendix (with generally very similar results). We emphasize that our method does not require any knowledge of the treatment effects during policy optimization, so different choices of $\beta$ or structures for treatment effect heterogeneity only impact the sample size required to detect treatment effects.

% \paragraph{Simulation setup} For each dataset, the utility function $u$ is set to be the risk of an adverse outcome absent treatment, estimated via a random forest. For the first domain, the outcome is one-year reentry and for the second the outcome is having at least  four emergency department visits in the following year. We then simulated a hypothetical intervention with specified budget $b$ and a treatment effect given by $\tau(X) = \beta \cdot \Pr(Y(0) = 1|X)$, where $\beta$ is the relative reduction in the frequency of adverse outcomes caused by the treatment and $\Pr(Y(0) = 1|X)$ is estimated via the random forest. Our main results are shown for $\beta = 0.1$. Results for other values are qualitatively similar and are given in the appendix, along with details of the setup. We remark that this simulated treatment effect is heterogeneous in the sense often hypothesized by service designers, where treatment effects are largest for those with highest risk (largest $\Pr(Y(0) = 1|X)$). 

\paragraph{Evaluation metrics and baselines} We evaluate our designs on two metrics. First, utility from targeting, presented in terms of the fraction of individuals with an adverse baseline outcome who would be allocated treatment (i.e., a measure of recall or sensitivity). The second metric is the variance with which we can estimate the ATE. To provide a concrete interpretation for the variance, we present it in terms of the sample size needed to power an estimate. Specifically, we assume standard normality-based confidence intervals and find the sample size needed to detect an effect of at least $\tau$ with 5\% type-1 error and 80\% power. We benchmark our designs against two policy-relevant points. First, the sample size needed to power a standard RCT that randomizes the same treatment budget uniformly across individuals. Second, the sample size needed to power a regression discontinuity (RD) estimate of the \textit{local} average treatment effect around the budget cutoff. While the RD is not directly comparable as it estimates a different effect, it provides a reference point that is commonly used when policies like need-based targeting create clear discontinuities. 

To calculate sample sizes for the regression discontinuity (RD) estimates, we follow a procedure based on \citet{schochet2009statistical} and \citet{deke2012statistical}. This requires two steps. First, we use the distribution of $u(X)$ to estimate the design effect of a standard OLS-based RD estimator compared to a RCT. Second, we estimate the optimal bandwidth for a locally linear estimator using the method of \citet{imbens2012optimal} in order to determine what fraction of the sample would be used. There are two important caveats for this analysis. First, as discussed above, the RD targets a different and incomparable estimand (the treatment effect for individuals at the margin, versus the average effect). Second, the RD returns a potentially biased estimate, as it uses a locally linear model with a bandwidth that optimizes an empirical bias-variance tradeoff for the MSE. By contrast, estimates derived from the designs output by our method are guaranteed to be unbiased (as the propensity scores are known) and so accrue error only due to variance. We ignore the RD estimator's bias in power calculations, erring towards slightly optimistic results for the RD.  

To contextualize the gains from optimization in this setting, we also compare to two heuristic ways of setting treatment probabilities such that individuals at higher risk receive more treatment. In the first (``scaling") we set $p(X_i) \propto u(X_i)^{\alpha}$ for some parameter $\alpha > 0$. For the second (``softmax"), we set $p(X_i) \propto \exp(\alpha \cdot u(X_i))$. In both cases, $\alpha$ is a temperature parameter where larger values concentrate more probability on higher-risk individuals and $\alpha = 0$ recovers uniform allocation. We vary $\alpha$ to obtain a Pareto frontier for each baseline that can be compared to that output by our method. For both methods, obtaining valid allocation probabilites requires that we somehow clip the values to [0, 1] and normalize to obey the budget constraint. We implement a simple heuristic that successively normalizes by $||p||_1$ and clips the entries to be at most 0.99. If after clipping some of the budget is left over, we repeat the process on all of the non-clipped entries. We emphasize though that both of these methods are not full-fledged competitors in multiple senses. First, they are unable to accommodate additional constraints on the policy class (e.g., fairness constraints). Second, they require that the entire cohort of eligible individuals is present at one time, as opposed to handling online arrivals like our method. However, real service delivery settings almost always entail online arrivals.  

\subsection{Simulation results}
\begin{figure}
    \centering
    \includegraphics[width=1.8in]{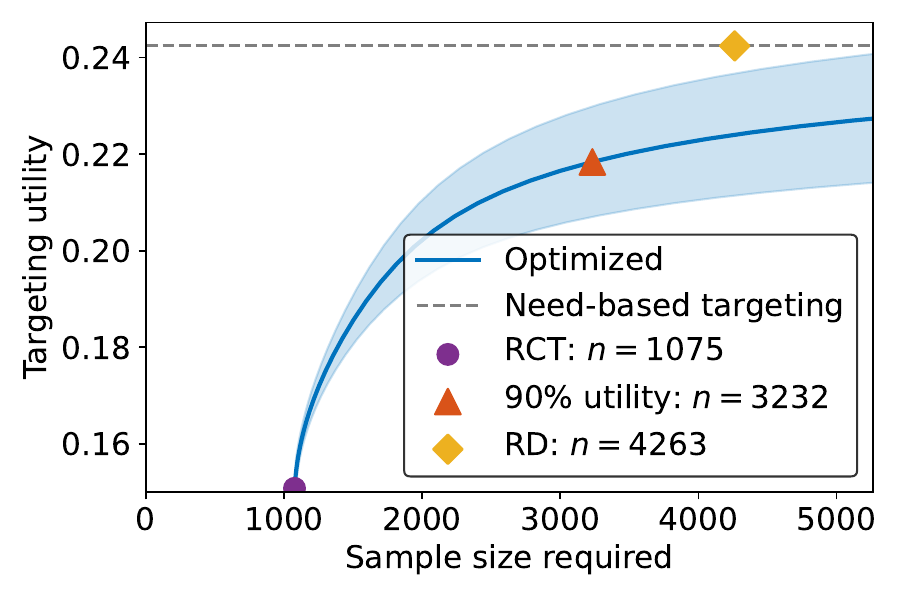}
    \includegraphics[width=1.8in]{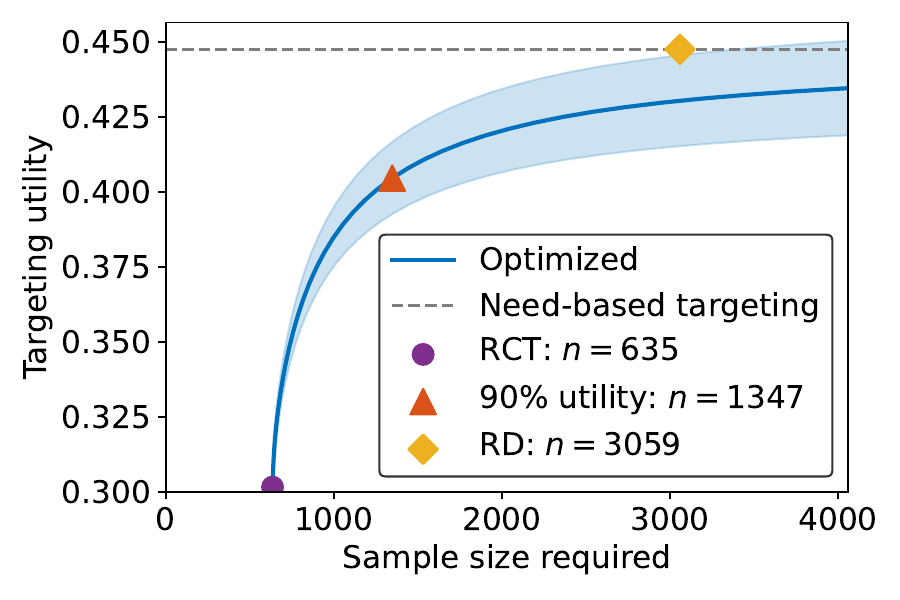}
    \includegraphics[width=1.8in]{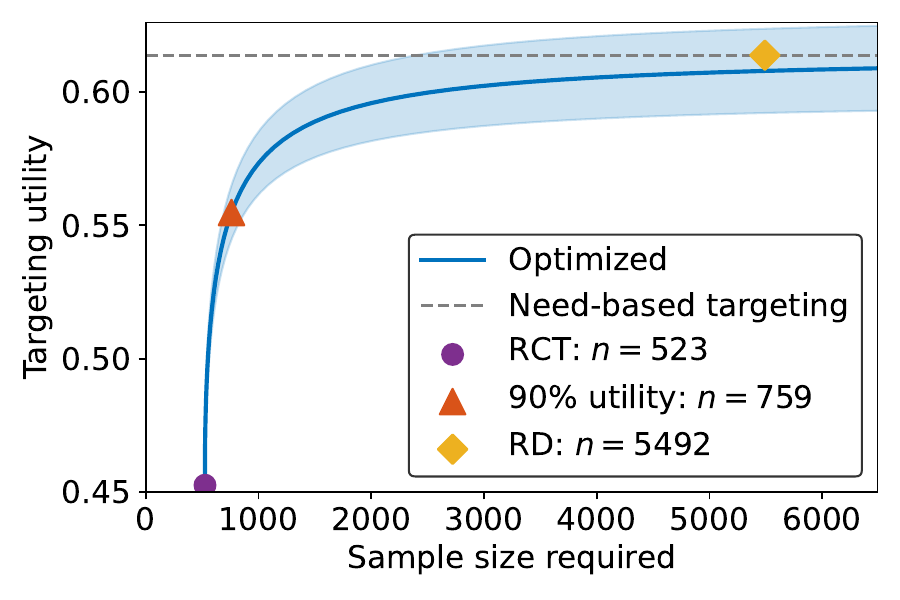}
    \includegraphics[width=1.8in]{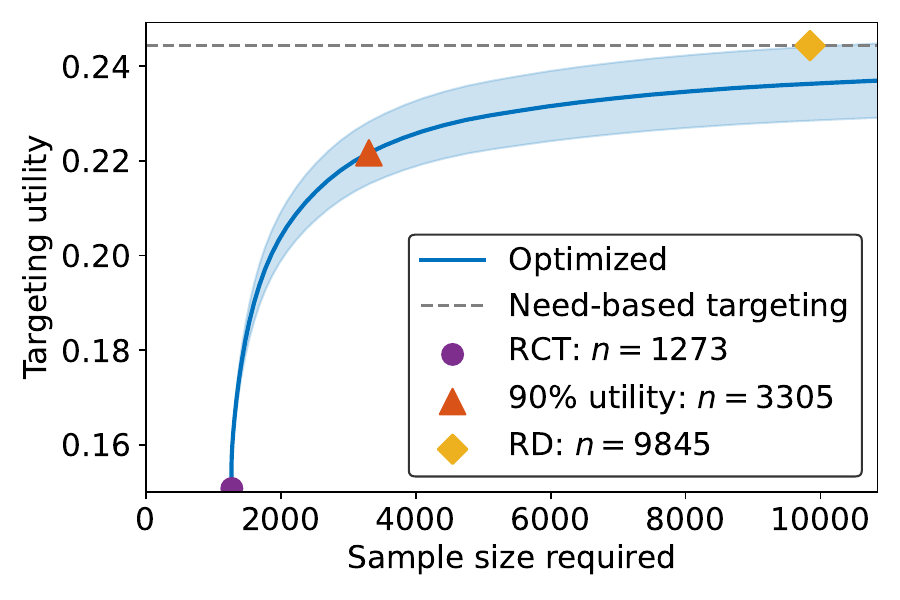}
    \includegraphics[width=1.8in]{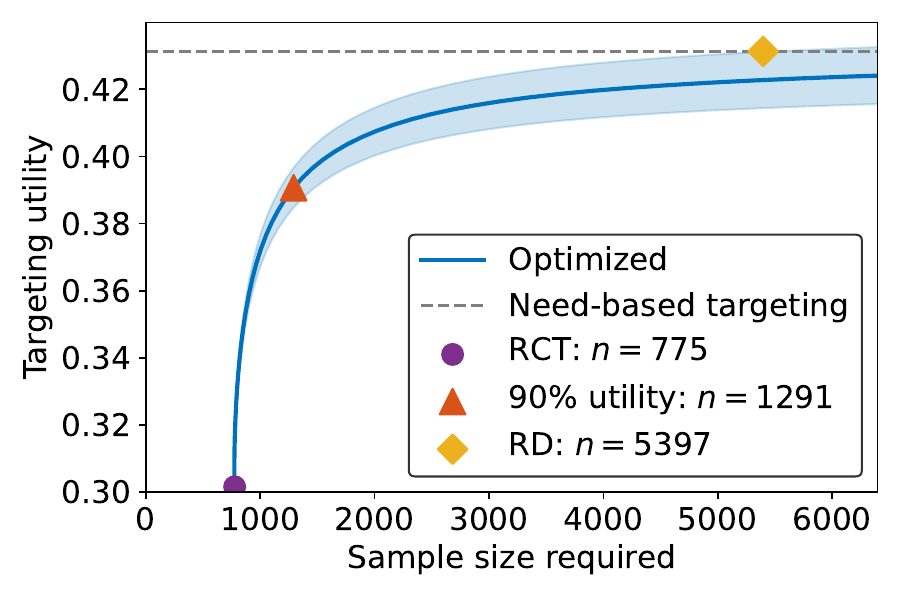}
    \includegraphics[width=1.8in]{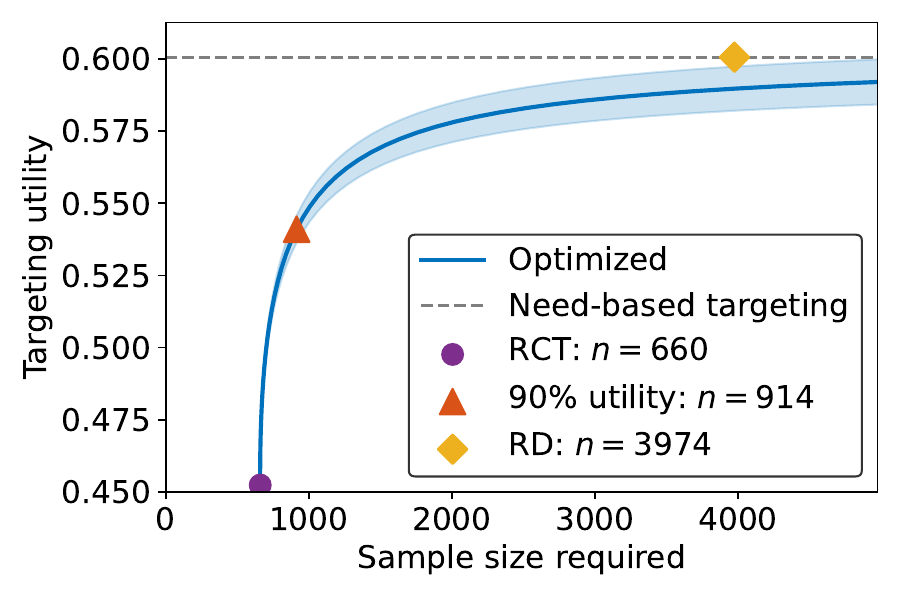}
    \caption{Utility vs sample size required to estimate the average treatment effect. Top row: housing dataset. Bottom row: reentry dataset. Columns from left to right: $b = 0.15, 0.3, 0.45$. Each blue curve is produced by varying the value of the utility constraint, solving for the corresponding optimal policy at each point, and calculating the variance of the corresponding ATE estimate (presented in terms of sample size required to power an estimate). The purple dot shows the sample size and utility of running a RCT which randomizes individuals uniformly with the same budget. The gray dashed line shows the utility of need-based targeting (without randomization). The gold diamond shows the sample size required to power a regression discontinuity estimate of the local ATE after need-based targeting. The red triangle shows the point on the curve at which the optimized policy obtains 90\% of the best-possible utility.}
    \label{fig:main-results}
\end{figure}

\paragraph{Primary results} Figure \ref{fig:main-results} gives our primary simulation result: the Pareto frontier between sample size and targeting utility as computed by our method. For our primary analysis, we use the conservative variance model in Example 1, since it does not require assuming any prior knowledge about the outcome variances and is hence the most widely applicable. Later, we examine whether gains are possible from using more knowledge about the variance structure. Every point on the curve corresponds to a specific set of assignment probabilities output as the solution to an instance of the optimization problem detailed above. The curve is traced out by varying the value of the constraint on utility. Shaded regions give 95\% confidence bands computed using the multiplier bootstrap \citep{van1996weak,kennedy2019nonparametric} with 10,000 replicates. We observe that the tradeoff between these two objectives is relatively favorable, with a strongly concave shape to the curve. That is, most of the gap in utility between the RCT and need-based targeting can be made up with relatively small costs to sample size. Intuitively, this behavior is expected because utility increases linearly with $p(X)$ while variance increases much faster than linearly as $p(X)$ approaches 1. Quantitatively, the optimized policy can achieve 90\% of the best possible utility using approximately 1.5-3 times the samples required for a RCT.  By comparison, we find that the RD requires 4-10 times the sample size of a RCT, and often over 3 times that of the ``90\% utility" point for our method. This is a consequence of the fact that the RD can only make use of samples relatively close to the decision boundary. Accordingly, while a RD may be attractive for a policymaker who places a very high premium on targeting, it imposes significant costs in terms of sample size requirements. 

In Figure \ref{fig:heuristics}, we compare to the two heuristic strategies for scaling allocation probabilities with utility. Both generally underperform the optimized probabilities. Across the six settings pictured (two datasets and three values of the budget constraint), the heuristic strategies come close to the optimum in one. Across the others, they underperform by varying amounts, with little consistent pattern across settings. In many settings, achieving 90\% of the best possible utility would require very large sample sizes (past the end of the plot) for one or both of the heuristics. We conclude that systematic optimization offers stronger and more robust performance across varied data distributions and budgets.

\begin{figure}
    \centering
    \includegraphics[width=1.8in]{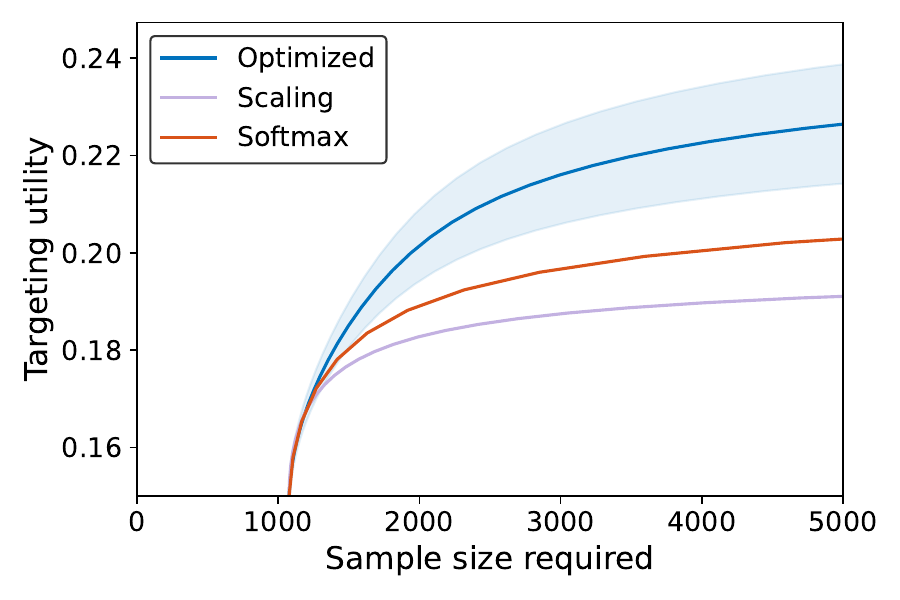}
    \includegraphics[width=1.8in]{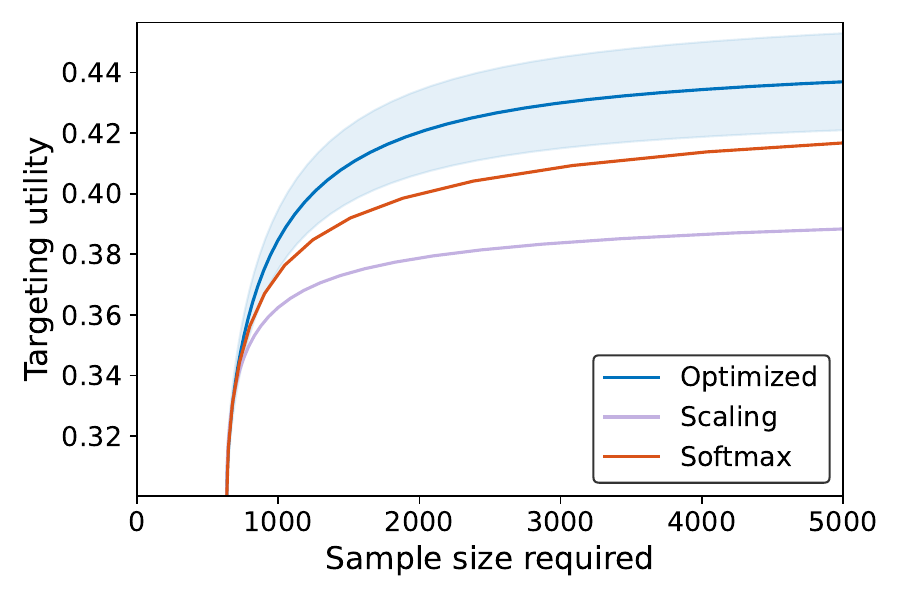}
    \includegraphics[width=1.8in]{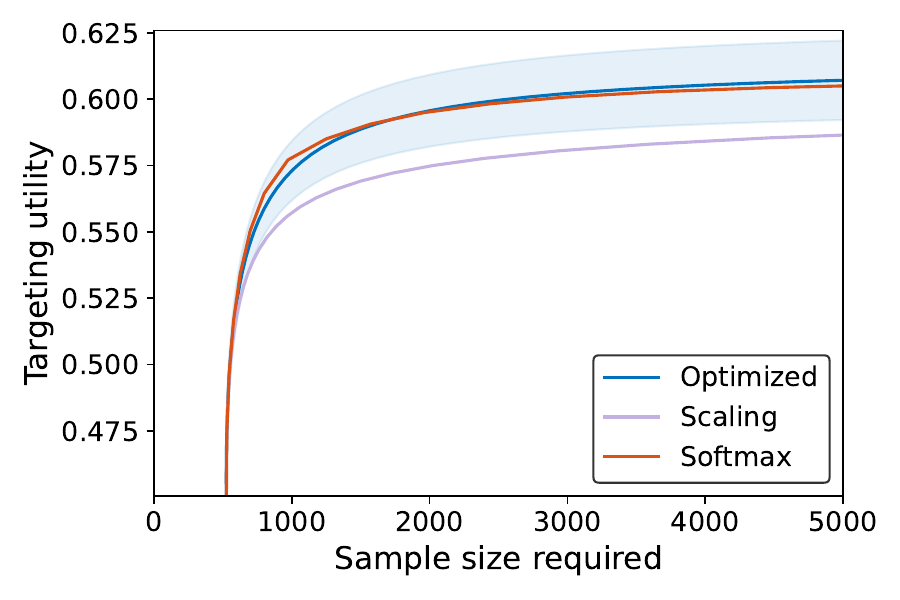}
    \includegraphics[width=1.8in]{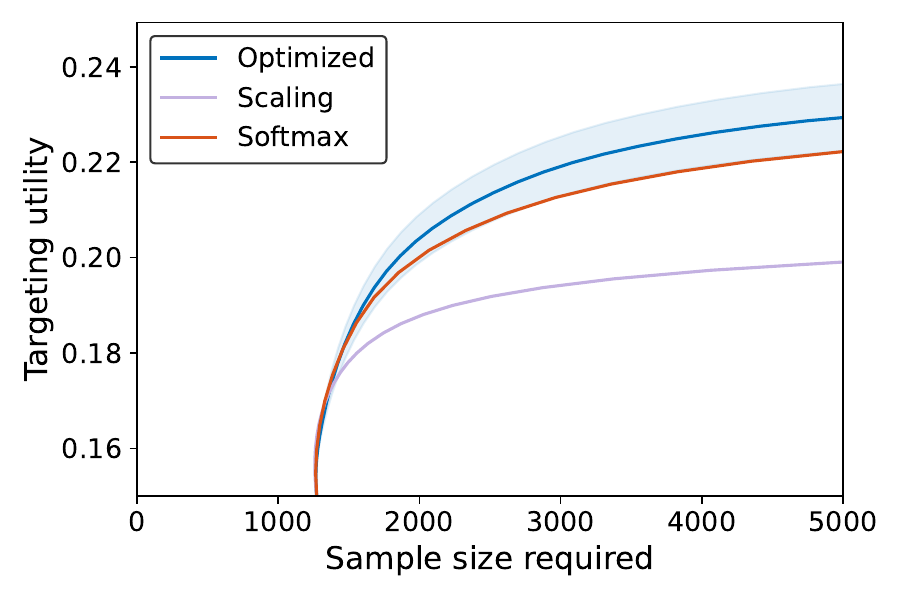}
    \includegraphics[width=1.8in]{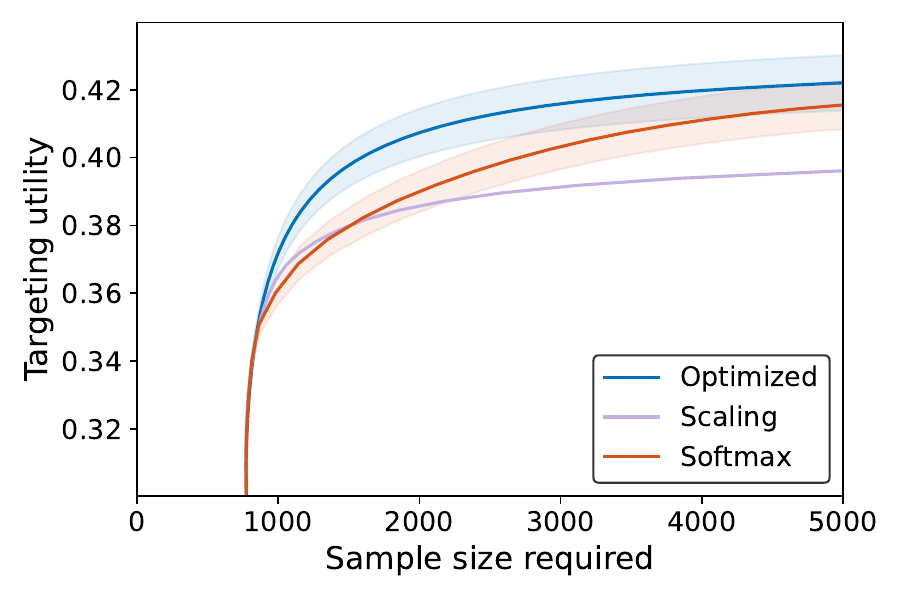}
    \includegraphics[width=1.8in]{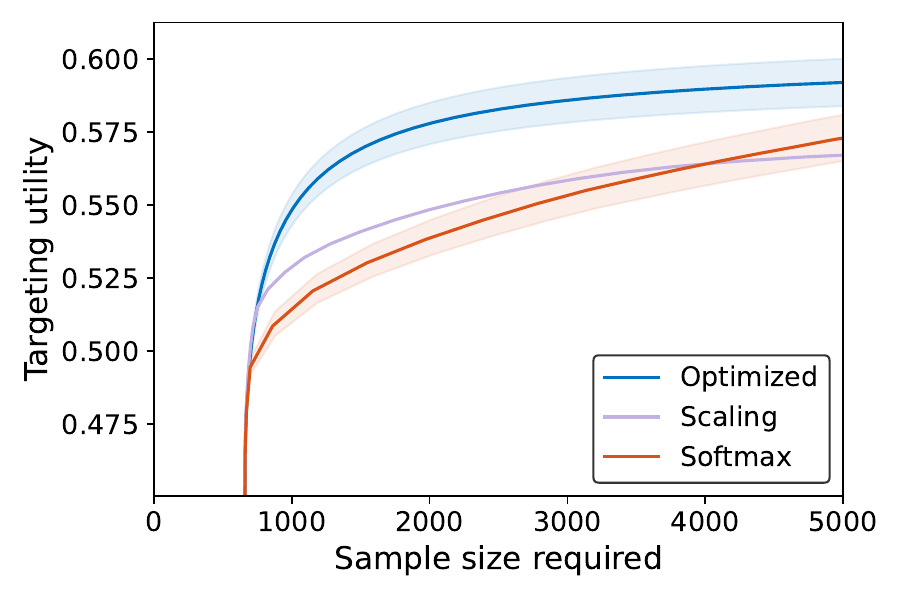}
    \caption{Comparison to heuristics for experimental design. Top row: housing data. Bottom row: reentry data. Columns from left to right: $b = 0.15, 0.30, 0.45$. } \label{fig:heuristics}
\end{figure}

% If a the policymaker is willing to apply our proposed designs while incurring the same sample size as the RD, utility typically suffers by 3-5\% (in relative terms). This can be viewed as the price of identifying a non-local estimand instead of the marginal effect estimated by the RD. 

\begin{figure}
    \centering
    \includegraphics[width=2.5in]{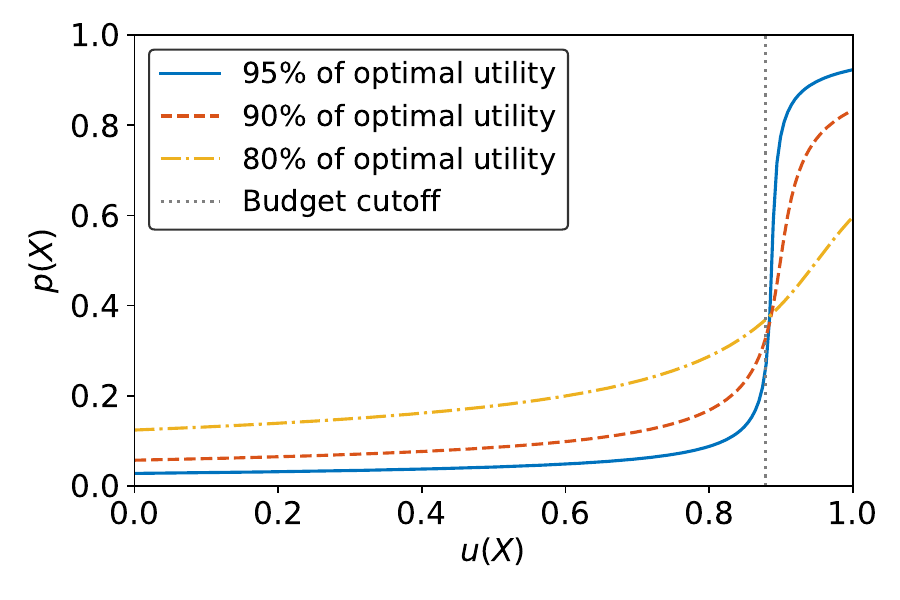}
    \includegraphics[width=2.5in]{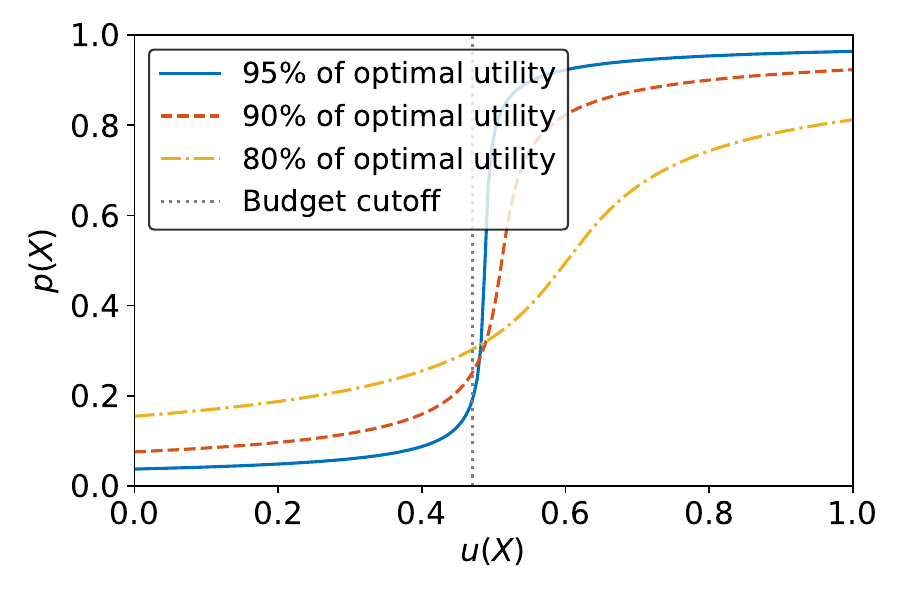}
    \caption{Visualization of the optimal policy at varying levels of the utility constraint. Left: housing dataset. Right: reentry dataset. Each curve plots the optimal treatment assignment probability $p(X)$ as a function of the individual's utility score $u(X)$ (for the specified value of the overall utility constraint). The vertical dashed line gives the $(1 - b)$-quantile of $u$, where $b$ is the fraction of individuals treated. Both plots are for $b = 0.3$.}
    \label{fig:example-policies}
\end{figure}

\paragraph{Visualizing optimal policies} Figure \ref{fig:example-policies} shows examples of the policies output by our method at varying levels of the utility constraint, plotting the optimal $p(X)$ as a function of an individual's utility value $u(X)$. The optimal policies take an intuitive S-shaped curve with steepness that depends on the utility constraint. At larger values of the constraint, the optimal policy effectively redistribute resources from individuals whose utility value lies just over the cutoff for need-based targeting to individuals in the region below the cutoff.

\begin{figure}
    \centering
    \includegraphics[width=2.5in]{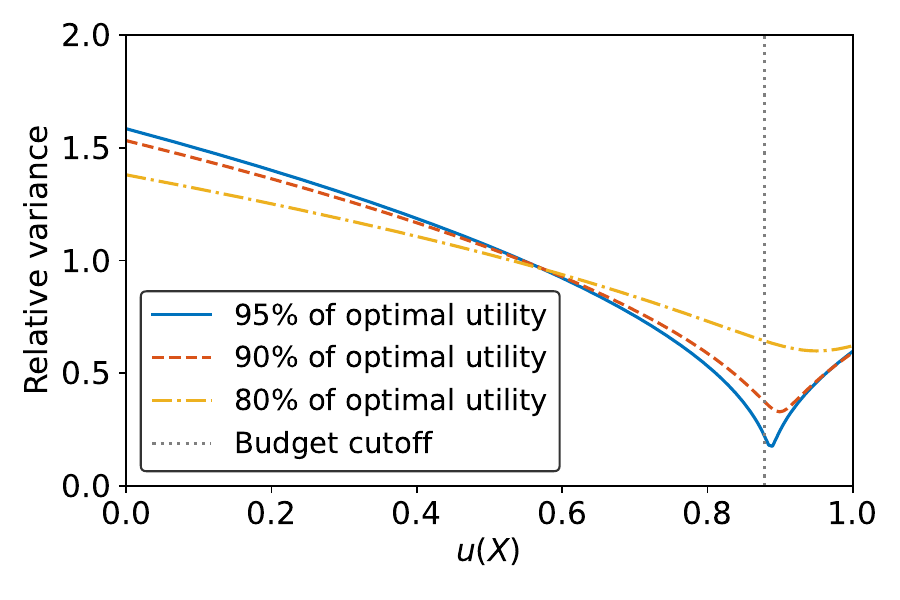}
    \includegraphics[width=2.5in]{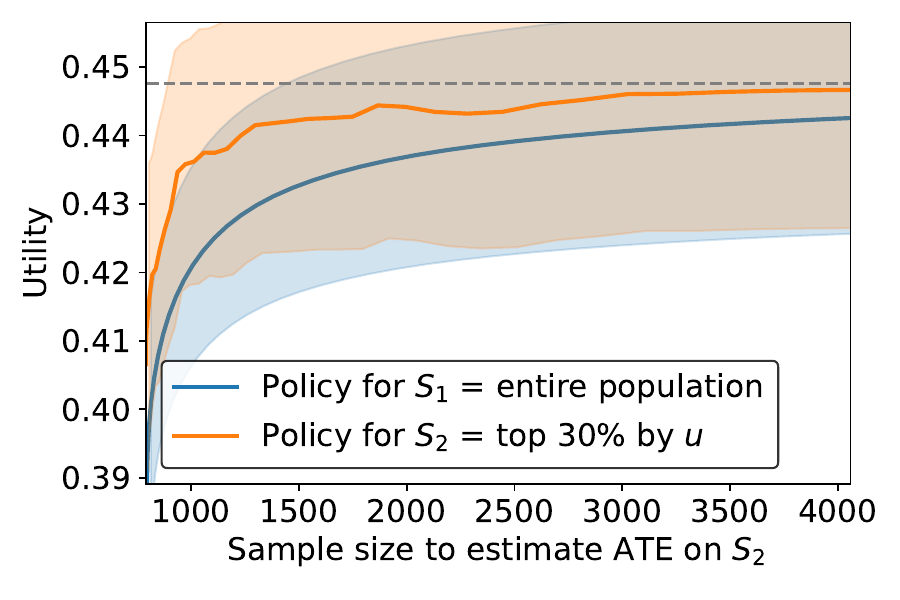}
    \includegraphics[width=2.5in]{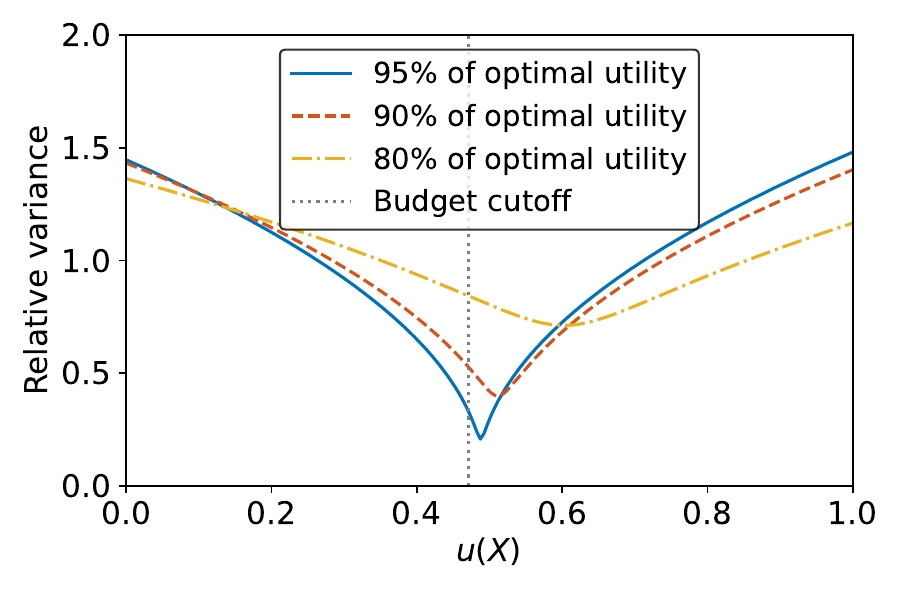}
    \includegraphics[width=2.5in]{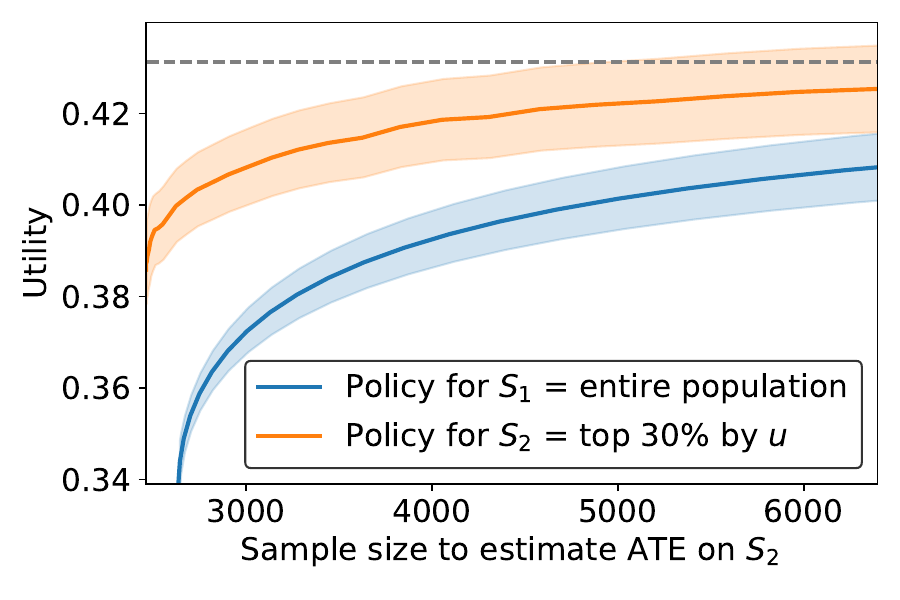}
    \caption{Ability to estimate heterogeneous treatment effects for groups defined by risk strata. Top: housing data. Bottom: reentry data. Left: relative variance of ATE estimation at each level of $u$ compared to the whole-population mean. Right: simulation of a setting where we aim to estimate the ATE for the 30\% highest-risk individuals. The blue curve gives the sample size-utility tradeoff for the original policy which optimizes estimation of the ATE for the whole population, while the orange curve gives the same tradeoff for a policy optimized to estimate the ATE just on the high-risk subgroup.}
    \label{fig:att}
\end{figure}

\paragraph{Statistical power for subgroups} Our results so far have aimed at estimation of the ATE on the entire population. We next turn to the ability to estimate effects for more specific subpopulations. Figure \ref{fig:att} shows the \textit{pointwise} value of the efficiency bound for each of the same optimized policies as a function of the utility $u(X)$. The variance of a group average treatment effect is obtained by integrating this curve over the population of interest so that the whole-population ATE has variance given by the integral over the entire population. We normalize the curve so that the whole-population ATE has value 1, allowing us to interpret the curve as the relative increase or decrease in sampling variance if we tried to estimate the group average treatment effect (GATE) over a particular subset of $\mathcal{X}$. We find that our designs have greatest power to estimate GATEs for levels of the utility function near the budget cutoff, while evidence regarding GATEs at one extreme or another would accrue at a rate of up to 1.5 times slower than the mean. If a policymaker specifically hopes to estimate a particular GATE, e.g. for higher-need individuals, they can modify the objective function as described in Section \ref{section:alternate-estimands}. The right-hand column of Figure \ref{fig:att} shows an example of this, where we simulate the goal of estimating the GATE for the target population $\mathcal{S}_2 = \{x: u(x) \geq \text{quantile}_{1 - b}(u)\}$, denoting the $b$-highest need subgroup. We compare the utility-sample size tradeoff curve for the policy optimized with respect to $\mathcal{S}_2$ compared to the original policy, which optimizes for estimation of the whole-population ATE ($\mathcal{S}_1$). If the policymaker is willing to commit to the more specific estimand, they are able to gain additional utility by forgoing exploration on the remainder of the population. 

% Figure \ref{fig:att} examples the consequences of specifying an alternate estimand instead of the whole-population average treatment effect. Specifically, we examine the case where a policymaker only wishes to know whether an intervention is effective for the highest-need subgroup. We formalize this as in Section \ref{section:alternate-estimands} by comparing estimation for the target population $\mathcal{S}_1 = \mathcal{X}$ (i.e., the entire population) to $\mathcal{S}_2 = \{x: u(x) \geq \text{quantile}_{1 - b}(u)\}$ denotes the $b$-highest risk subgroup. We compute the optimal policies implied by each of these estimands, along with their corresponding utilities and sample size requirements (where the sample size calculation for the policy targeting $\mathcal{S}_2$ incorporates that only a $b$ fraction of the total samples will belong to the target population). The solid lines plot the two curves, and we observe that the utility consequences of the choice of estimand vary depending on the data distribution. We also show, via the dashed line, performance of the optimal policy for the whole population ATE ($\mathcal{S}_1$) when evaluated in terms of the sample size required to estimate the ATE high-need population $\mathcal{S}_2$. Because this policy randomizes over the entire covariate space, such estimates are always possible. Depending on the data distribution however, it maybe possible to gain significantly in terms of the statistical power vs utility tradeoff by specifically optimizing for the $\mathcal{S}_2$ estimand if this is the primary object of interest. 

\begin{figure}
    \centering
    \includegraphics[width=2.5in]{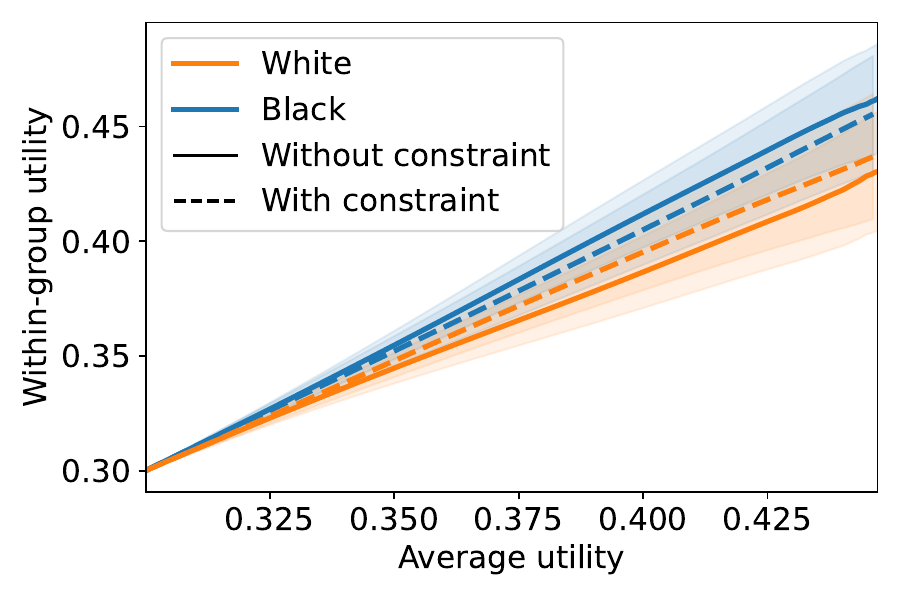}
    \includegraphics[width=2.5in]{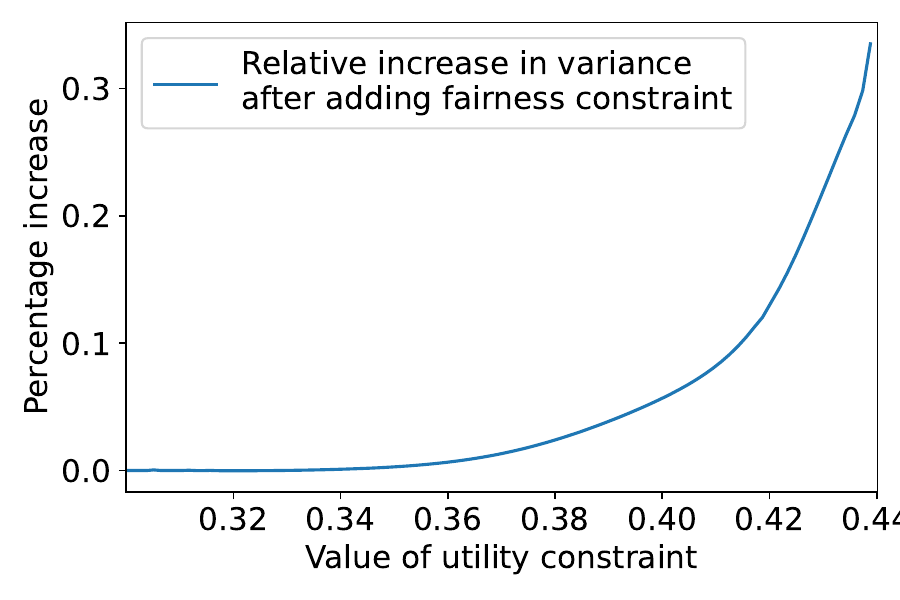}

    \includegraphics[width=2.5in]{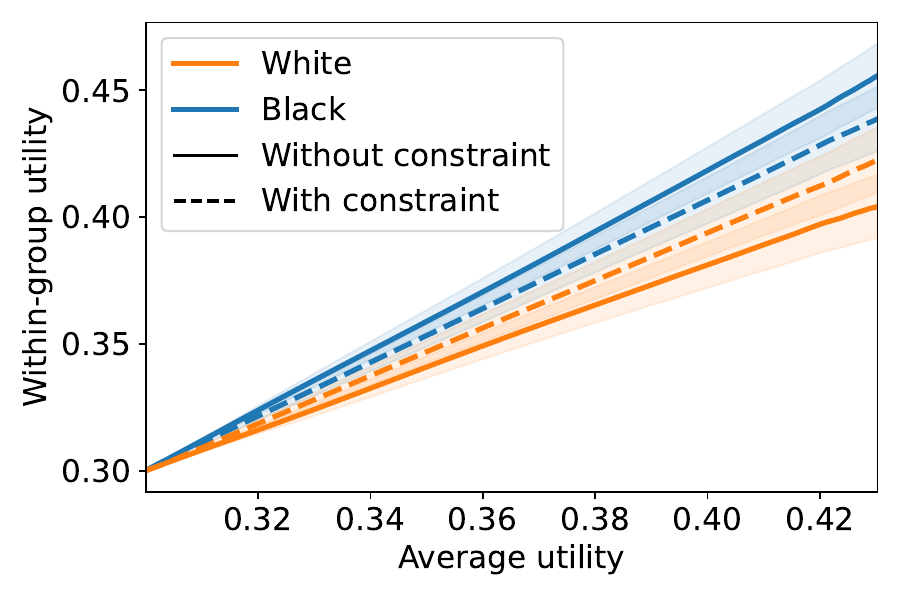}
    \includegraphics[width=2.5in]{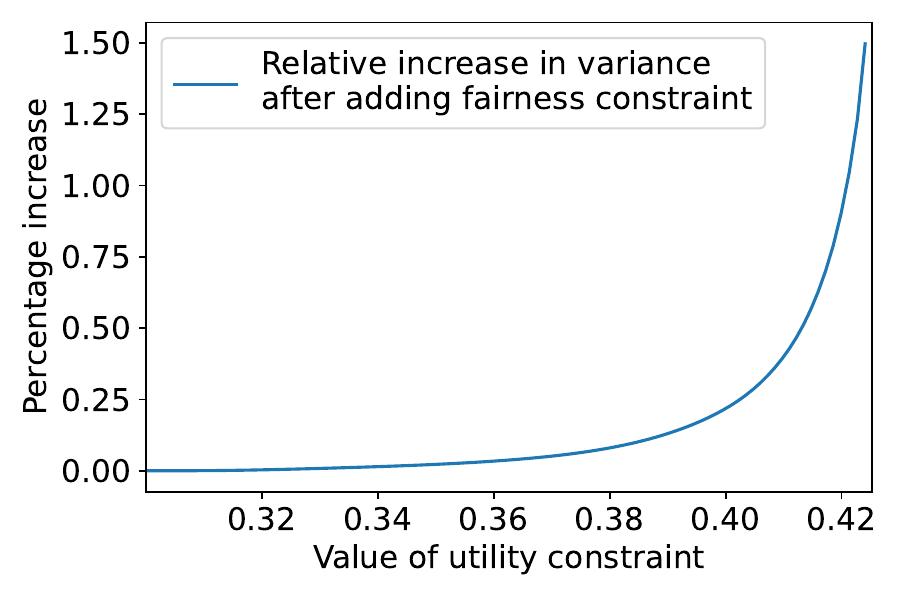}
    \caption{Impact of fairness constraints. Top: housing dataset. Bottom: reentry dataset. The left-hand plot of each row shows the utility of the optimal policies for each group, i.e., the fraction of members of each group with an adverse outcome who receive treatment. The $x$ axis varies the average utility across the entire population (i.e., the value of the utility constraint). Solid lines show policies without fairness constraints, while dashed lines show policies with fairness constraints. The right-hand column shows the percentage increase in variance (i.e., the objective function) that results from imposing fairness constraints, again as a function of the value of the utility constraint.}
    \label{fig:fairness}
\end{figure}

\paragraph{Impact of fairness constraints} Finally, Figure \ref{fig:fairness} shows the consequences of adding fairness constraints (based on Equation \ref{eq:fairness-constraint}) to the optimization problem. We take the groups to consist of White and Black individuals. We observe that the optimal policies without fairness constraints imply slightly higher average utility for Black individuals than White. Imposing fairness constraints narrows the difference, at only a nominal cost in terms of estimation variance. The small cost of imposing fairness constraints is attributable to the fact that the two groups have relatively similar distributions of the utility function in our dataset. In other settings, a larger ``price of fairness" could arise if a decision maker specifies a utility function which takes systematically larger values on one of the groups.   

\begin{figure}
    \centering
    \includegraphics[width=2.5in]{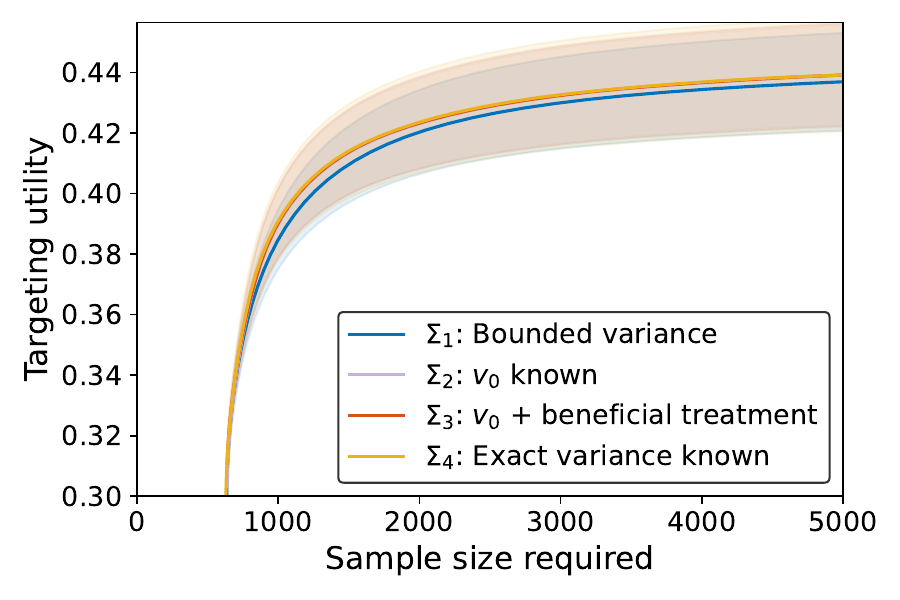}
    \includegraphics[width=2.5in]{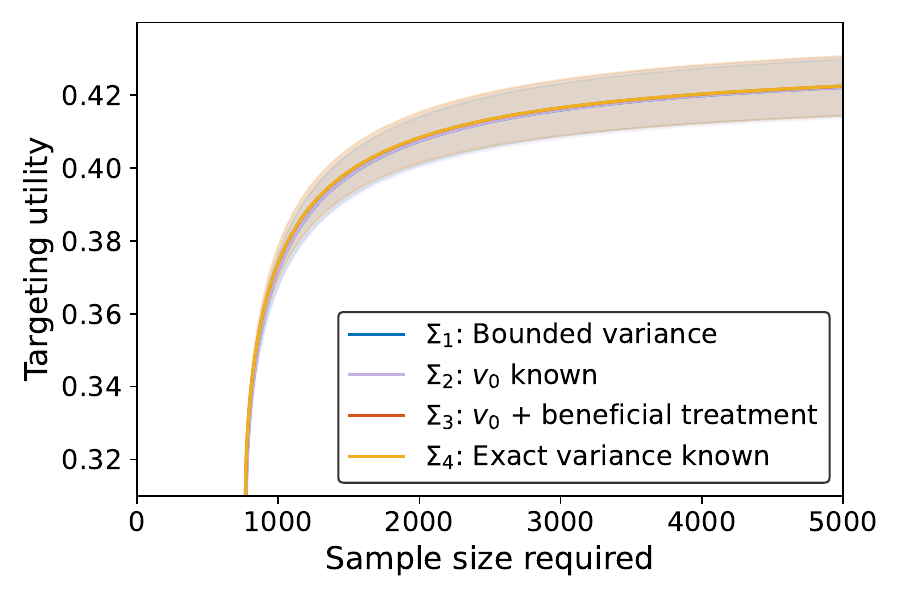}
    \caption{Comparison of the recall-sample size tradeoff for policies optimized under the four uncertainty sets $\Sigma_1$-$\Sigma_4$. Left: housing data. Right: reentry data.} \label{fig:variance-models}
\end{figure}
\paragraph{Impact of knowledge of the outcome variance} So far, our results have all been for optimizing with respect to the most conservative uncertainty set for the outcome variances, which requires no prior knowledge. We now simulate a series of policies that optimize against progressively better-informed uncertainty sets. At one extreme is the default policy used in results so far, which optimizes only under the assumption that the variances are bounded. We refer to this uncertainty set as $\Sigma_1$. Second, an uncertainty set $\Sigma_2$ where we assume that $v_0$ is known (estimated from the same historical data used to learn the risk model) while $v_1$ is assumed to be bounded by $\frac{1}{4}$ (the maximum value for a binary outcome). Third, $\Sigma_3$, a refinement of $\Sigma_2$ where we assume that $\tau(X) \leq 0$ (as holds for the simulated treatment effect) and so impose the additional constraint that $v_1(X) \leq v_0(X)$ whenever $\Pr(Y(0) = 1|X) \leq 0.5$. Finally, $\Sigma_4$, the (unattainable) optimal case where both $v_0$ and $v_1$ are perfectly known and constrained to lie at the true values. Each step from $\Sigma_1$ and $\Sigma_4$ adds progressively stronger (correct) constraints on the outcome variances, allowing us to test whether having greater knowledge of the variances improves performance in practice. 

The results are shown in Figure \ref{fig:variance-models} (simulated with our default settings of $b = 0.3$ and $\beta = 0.1)$. We find that in both datasets, there is almost no gain from having additional information about the outcome variances. In the housing dataset, recall gains are on order of a 0.5\% relative improvement, while there is no discernible impact for the reentry dataset. The gains on the housing dataset are achievable in any of the models where $v_0$ is known ($\Sigma_2$-$\Sigma_4$) compared to $\Sigma_1$, indicating that the ability to estimate the variance of the baseline outcome $Y(0)$ is sufficient to capture most of the gains (at least for this empirical setting). From these results, we conclude that optimal assignment policies are quite robust to lack of knowledge of the outcome variances, at least in application domains like the ones that we study. It is possible that adaptation to the variance structure could be more important in settings with much greater imbalance in variance between arms of the study.

\section{Discussion}

Designing effective interventions in socially critical domains often requires policymakers to confront difficult tradeoffs between competing objectives. Chief among these is the need to balance offering the highest-quality services in the present against the potential to learn more and improve in the future. This paper investigates a specific instantiation of this dilemma, in the form of balancing between treating individuals who are judged to have higher need against the ability to learn about the average treatment effect of the program. We provide algorithms to minimize the error of ATE estimates subject to a constraint on the expected utility of the policy in treating high-need individuals.  We find that by explicitly optimizing the tradeoff between these competing goals, there is substantial room to navigate between the all-or-nothing extremes of standard RCTs and targeting purely based on estimated need. 

In the future, such issues may grow in importance due to the increasing use of predictive models in policy, healthcare, and social service settings. These models are typically trained to classify baseline outcomes as accurately as possible. The availability of such a model creates a strong impetus to allocate treatment to the individuals deemed at highest risk. However, our results can be seen as a caution that pure predictive accuracy is not the only salient goal for such systems: giving up a marginal amount of alignment between allocations and risk prediction can create large benefits in terms of other objectives for program design (in our case, rigorous evaluation). In this sense, our results dovetail with \citet{jain2024scarce}'s normative argument that resource allocations based on machine learning predictions should be randomized. Our approach can be seen as giving the most effective way to implement such randomization when the goal is to control the tradeoff between utility and estimating treatment effects. 

While risk prediction can help improve the efficiency with which limited resources are used, prediction is ultimately only one component of successful interventions. Evaluations of overall program effectiveness are highly valuable to guide larger-scale strategic decisions about how to design program offerings in the first place. Indeed, such higher-level decisions may be more consequential for overall impact than refining individual-level targeting decisions \citep{perdomo2024relative}. A broader conclusion from our work is that the design of algorithmic systems for resource allocation should explicitly account for such goals beyond predictive accuracy. Constrained optimization frameworks like ours provide a means to build in additional objectives which reflect the role that predictive models play in a larger system. Accounting for this larger picture can help algorithm design play a more fruitful role in the provision of public services.

\begin{acks}
BW acknowledges support from the AI2050 program at \grantsponsor{G-22-64474}{Schmidt Sciences}{} (Grant \#\grantnum{G-22-64474}{G-22-64474}) and the AI Research Institutes Program funded by the \grantsponsor{2229881}{National Science Foundation}{} under AI Institute for Societal Decision Making (AI-SDM), Award No. \grantnum{2229881}{2229881}. BW also thanks participants at the BIRS Workshop on Bridging Prediction and Intervention Problems in Social Systems for fruitful conversations.
\end{acks}

% Bibliography
\bibliographystyle{ACM-Reference-Format}
\bibliography{refs}


\begin{thebibliography}{42}

%%% the Web address.

\ifx \showCODEN    \undefined \def \showCODEN     #1{\unskip}     \fi
\ifx \showDOI      \undefined \def \showDOI       #1{#1}\fi
\ifx \showISBNx    \undefined \def \showISBNx     #1{\unskip}     \fi
\ifx \showISBNxiii \undefined \def \showISBNxiii  #1{\unskip}     \fi
\ifx \showISSN     \undefined \def \showISSN      #1{\unskip}     \fi
\ifx \showLCCN     \undefined \def \showLCCN      #1{\unskip}     \fi
\ifx \shownote     \undefined \def \shownote      #1{#1}          \fi
\ifx \showarticletitle \undefined \def \showarticletitle #1{#1}   \fi
\ifx \showURL      \undefined \def \showURL       {\relax}        \fi
% invisible to TeX
\providecommand\bibfield[2]{#2}
\providecommand\bibinfo[2]{#2}
\providecommand\natexlab[1]{#1}
\providecommand\showeprint[2][]{arXiv:#2}

\bibitem[Agarwal et~al\mbox{.}(2014)]%
        {agarwal2014taming}
\bibfield{author}{\bibinfo{person}{Alekh Agarwal}, \bibinfo{person}{Daniel Hsu}, \bibinfo{person}{Satyen Kale}, \bibinfo{person}{John Langford}, \bibinfo{person}{Lihong Li}, {and} \bibinfo{person}{Robert Schapire}.} \bibinfo{year}{2014}\natexlab{}.
\newblock \showarticletitle{Taming the monster: A fast and simple algorithm for contextual bandits}. In \bibinfo{booktitle}{\emph{International Conference on Machine Learning}}. PMLR, \bibinfo{pages}{1638--1646}.
\newblock


\bibitem[Aiken et~al\mbox{.}(2022)]%
        {aiken2022machine}
\bibfield{author}{\bibinfo{person}{Emily Aiken}, \bibinfo{person}{Suzanne Bellue}, \bibinfo{person}{Dean Karlan}, \bibinfo{person}{Chris Udry}, {and} \bibinfo{person}{Joshua~E Blumenstock}.} \bibinfo{year}{2022}\natexlab{}.
\newblock \showarticletitle{Machine learning and phone data can improve targeting of humanitarian aid}.
\newblock \bibinfo{journal}{\emph{Nature}} \bibinfo{volume}{603}, \bibinfo{number}{7903} (\bibinfo{year}{2022}), \bibinfo{pages}{864--870}.
\newblock


\bibitem[Alatas et~al\mbox{.}(2012)]%
        {alatas2012targeting}
\bibfield{author}{\bibinfo{person}{Vivi Alatas}, \bibinfo{person}{Abhijit Banerjee}, \bibinfo{person}{Rema Hanna}, \bibinfo{person}{Benjamin~A Olken}, {and} \bibinfo{person}{Julia Tobias}.} \bibinfo{year}{2012}\natexlab{}.
\newblock \showarticletitle{Targeting the poor: evidence from a field experiment in Indonesia}.
\newblock \bibinfo{journal}{\emph{American Economic Review}} \bibinfo{volume}{102}, \bibinfo{number}{4} (\bibinfo{year}{2012}), \bibinfo{pages}{1206--1240}.
\newblock


\bibitem[Athey et~al\mbox{.}(2023)]%
        {athey2023machine}
\bibfield{author}{\bibinfo{person}{Susan Athey}, \bibinfo{person}{Niall Keleher}, {and} \bibinfo{person}{Jann Spiess}.} \bibinfo{year}{2023}\natexlab{}.
\newblock \showarticletitle{Machine learning who to nudge: causal vs predictive targeting in a field experiment on student financial aid renewal}.
\newblock \bibinfo{journal}{\emph{arXiv preprint arXiv:2310.08672}} (\bibinfo{year}{2023}).
\newblock


\bibitem[Athey and Wager(2021)]%
        {athey2021policy}
\bibfield{author}{\bibinfo{person}{Susan Athey} {and} \bibinfo{person}{Stefan Wager}.} \bibinfo{year}{2021}\natexlab{}.
\newblock \showarticletitle{Policy learning with observational data}.
\newblock \bibinfo{journal}{\emph{Econometrica}} \bibinfo{volume}{89}, \bibinfo{number}{1} (\bibinfo{year}{2021}), \bibinfo{pages}{133--161}.
\newblock


\bibitem[Chassang et~al\mbox{.}(2012)]%
        {chassang2012selective}
\bibfield{author}{\bibinfo{person}{Sylvain Chassang}, \bibinfo{person}{Gerard Padr{\'o}~i Miquel}, {and} \bibinfo{person}{Erik Snowberg}.} \bibinfo{year}{2012}\natexlab{}.
\newblock \showarticletitle{Selective trials: A principal-agent approach to randomized controlled experiments}.
\newblock \bibinfo{journal}{\emph{American Economic Review}} \bibinfo{volume}{102}, \bibinfo{number}{4} (\bibinfo{year}{2012}), \bibinfo{pages}{1279--1309}.
\newblock


\bibitem[Chernozhukov et~al\mbox{.}(2019)]%
        {chernozhukov2019semi}
\bibfield{author}{\bibinfo{person}{Victor Chernozhukov}, \bibinfo{person}{Mert Demirer}, \bibinfo{person}{Greg Lewis}, {and} \bibinfo{person}{Vasilis Syrgkanis}.} \bibinfo{year}{2019}\natexlab{}.
\newblock \showarticletitle{Semi-parametric efficient policy learning with continuous actions}.
\newblock \bibinfo{journal}{\emph{Advances in Neural Information Processing Systems}}  \bibinfo{volume}{32} (\bibinfo{year}{2019}).
\newblock


\bibitem[Dai et~al\mbox{.}(2024)]%
        {dai2024clip}
\bibfield{author}{\bibinfo{person}{Jessica Dai}, \bibinfo{person}{Paula Gradu}, {and} \bibinfo{person}{Christopher Harshaw}.} \bibinfo{year}{2024}\natexlab{}.
\newblock \showarticletitle{Clip-{OGD}: An experimental design for adaptive {Neyman} allocation in sequential experiments}.
\newblock \bibinfo{journal}{\emph{Advances in Neural Information Processing Systems}}  \bibinfo{volume}{36} (\bibinfo{year}{2024}).
\newblock


\bibitem[Deke and Dragoset(2012)]%
        {deke2012statistical}
\bibfield{author}{\bibinfo{person}{John Deke} {and} \bibinfo{person}{Lisa Dragoset}.} \bibinfo{year}{2012}\natexlab{}.
\newblock \showarticletitle{Statistical Power for Regression Discontinuity Designs in Education: Empirical Estimates of Design Effects Relative to Randomized Controlled Trials. Working Paper.}
\newblock \bibinfo{journal}{\emph{Mathematica Policy Research, Inc.}} (\bibinfo{year}{2012}).
\newblock


\bibitem[Diamond et~al\mbox{.}(2016)]%
        {diamond2016estimating}
\bibfield{author}{\bibinfo{person}{Alexis Diamond}, \bibinfo{person}{Michael Gill}, \bibinfo{person}{Miguel~Angel Rebolledo~Dellepiane}, \bibinfo{person}{Emmanuel Skoufias}, \bibinfo{person}{Katja Vinha}, {and} \bibinfo{person}{Yiqing Xu}.} \bibinfo{year}{2016}\natexlab{}.
\newblock \showarticletitle{Estimating poverty rates in target populations: An assessment of the simple poverty scorecard and alternative approaches}.
\newblock \bibinfo{journal}{\emph{World Bank Policy Research Working Paper}} \bibinfo{number}{7793} (\bibinfo{year}{2016}).
\newblock


\bibitem[Dimakopoulou et~al\mbox{.}(2019)]%
        {dimakopoulou2019balanced}
\bibfield{author}{\bibinfo{person}{Maria Dimakopoulou}, \bibinfo{person}{Zhengyuan Zhou}, \bibinfo{person}{Susan Athey}, {and} \bibinfo{person}{Guido Imbens}.} \bibinfo{year}{2019}\natexlab{}.
\newblock \showarticletitle{Balanced linear contextual bandits}. In \bibinfo{booktitle}{\emph{Proceedings of the AAAI Conference on Artificial Intelligence}}, Vol.~\bibinfo{volume}{33}. \bibinfo{pages}{3445--3453}.
\newblock


\bibitem[Foster and Rakhlin(2020)]%
        {foster2020beyond}
\bibfield{author}{\bibinfo{person}{Dylan Foster} {and} \bibinfo{person}{Alexander Rakhlin}.} \bibinfo{year}{2020}\natexlab{}.
\newblock \showarticletitle{Beyond {UCB}: Optimal and efficient contextual bandits with regression oracles}. In \bibinfo{booktitle}{\emph{International Conference on Machine Learning}}. PMLR, \bibinfo{pages}{3199--3210}.
\newblock


\bibitem[Grosh and Baker(1995)]%
        {grosh1995proxy}
\bibfield{author}{\bibinfo{person}{Margaret Grosh} {and} \bibinfo{person}{Judy~L Baker}.} \bibinfo{year}{1995}\natexlab{}.
\newblock \showarticletitle{Proxy means tests for targeting social programs}.
\newblock \bibinfo{journal}{\emph{Living standards measurement study working paper}}  \bibinfo{volume}{118} (\bibinfo{year}{1995}), \bibinfo{pages}{1--49}.
\newblock


\bibitem[Hahn(1998)]%
        {hahn1998role}
\bibfield{author}{\bibinfo{person}{Jinyong Hahn}.} \bibinfo{year}{1998}\natexlab{}.
\newblock \showarticletitle{On the role of the propensity score in efficient semiparametric estimation of average treatment effects}.
\newblock \bibinfo{journal}{\emph{Econometrica}} (\bibinfo{year}{1998}), \bibinfo{pages}{315--331}.
\newblock


\bibitem[Haushofer et~al\mbox{.}(2022)]%
        {haushofer2022targeting}
\bibfield{author}{\bibinfo{person}{Johannes Haushofer}, \bibinfo{person}{Paul Niehaus}, \bibinfo{person}{Carlos Paramo}, \bibinfo{person}{Edward Miguel}, {and} \bibinfo{person}{Michael~W Walker}.} \bibinfo{year}{2022}\natexlab{}.
\newblock \bibinfo{booktitle}{\emph{Targeting impact versus deprivation}}.
\newblock \bibinfo{type}{{T}echnical {R}eport}. \bibinfo{institution}{National Bureau of Economic Research}.
\newblock


\bibitem[Heller et~al\mbox{.}(2022)]%
        {heller2022machine}
\bibfield{author}{\bibinfo{person}{Sara~B Heller}, \bibinfo{person}{Benjamin Jakubowski}, \bibinfo{person}{Zubin Jelveh}, {and} \bibinfo{person}{Max Kapustin}.} \bibinfo{year}{2022}\natexlab{}.
\newblock \bibinfo{booktitle}{\emph{Machine learning can predict shooting victimization well enough to help prevent it}}.
\newblock \bibinfo{type}{{T}echnical {R}eport}. \bibinfo{institution}{National Bureau of Economic Research}.
\newblock


\bibitem[Henderson et~al\mbox{.}(2023)]%
        {henderson2023integrating}
\bibfield{author}{\bibinfo{person}{Peter Henderson}, \bibinfo{person}{Ben Chugg}, \bibinfo{person}{Brandon Anderson}, \bibinfo{person}{Kristen Altenburger}, \bibinfo{person}{Alex Turk}, \bibinfo{person}{John Guyton}, \bibinfo{person}{Jacob Goldin}, {and} \bibinfo{person}{Daniel~E Ho}.} \bibinfo{year}{2023}\natexlab{}.
\newblock \showarticletitle{Integrating reward maximization and population estimation: Sequential decision-making for Internal Revenue Service audit selection}. In \bibinfo{booktitle}{\emph{Proceedings of the AAAI Conference on Artificial Intelligence}}, Vol.~\bibinfo{volume}{37}. \bibinfo{pages}{5087--5095}.
\newblock


\bibitem[Imbens and Kalyanaraman(2012)]%
        {imbens2012optimal}
\bibfield{author}{\bibinfo{person}{Guido Imbens} {and} \bibinfo{person}{Karthik Kalyanaraman}.} \bibinfo{year}{2012}\natexlab{}.
\newblock \showarticletitle{Optimal bandwidth choice for the regression discontinuity estimator}.
\newblock \bibinfo{journal}{\emph{The Review of economic studies}} \bibinfo{volume}{79}, \bibinfo{number}{3} (\bibinfo{year}{2012}), \bibinfo{pages}{933--959}.
\newblock


\bibitem[Imbens and Wooldridge(2009)]%
        {imbens2009recent}
\bibfield{author}{\bibinfo{person}{Guido Imbens} {and} \bibinfo{person}{Jeffrey~M Wooldridge}.} \bibinfo{year}{2009}\natexlab{}.
\newblock \showarticletitle{Recent developments in the econometrics of program evaluation}.
\newblock \bibinfo{journal}{\emph{Journal of economic literature}} \bibinfo{volume}{47}, \bibinfo{number}{1} (\bibinfo{year}{2009}), \bibinfo{pages}{5--86}.
\newblock


\bibitem[Inoue et~al\mbox{.}(2023)]%
        {inoue2023machine}
\bibfield{author}{\bibinfo{person}{Kosuke Inoue}, \bibinfo{person}{Susan Athey}, {and} \bibinfo{person}{Yusuke Tsugawa}.} \bibinfo{year}{2023}\natexlab{}.
\newblock \showarticletitle{Machine-learning-based high-benefit approach versus conventional high-risk approach in blood pressure management}.
\newblock \bibinfo{journal}{\emph{International Journal of Epidemiology}} \bibinfo{volume}{52}, \bibinfo{number}{4} (\bibinfo{year}{2023}), \bibinfo{pages}{1243--1256}.
\newblock


\bibitem[Jain et~al\mbox{.}(2024)]%
        {jain2024scarce}
\bibfield{author}{\bibinfo{person}{Shomik Jain}, \bibinfo{person}{Kathleen Creel}, {and} \bibinfo{person}{Ashia Wilson}.} \bibinfo{year}{2024}\natexlab{}.
\newblock \showarticletitle{Scarce Resource Allocations That Rely On Machine Learning Should Be Randomized}. In \bibinfo{booktitle}{\emph{International Conference on Machine Learning}}.
\newblock


\bibitem[Joulani et~al\mbox{.}(2013)]%
        {joulani2013online}
\bibfield{author}{\bibinfo{person}{Pooria Joulani}, \bibinfo{person}{Andras Gyorgy}, {and} \bibinfo{person}{Csaba Szepesv{\'a}ri}.} \bibinfo{year}{2013}\natexlab{}.
\newblock \showarticletitle{Online learning under delayed feedback}. In \bibinfo{booktitle}{\emph{International conference on machine learning}}. PMLR, \bibinfo{pages}{1453--1461}.
\newblock


\bibitem[Kennedy(2019)]%
        {kennedy2019nonparametric}
\bibfield{author}{\bibinfo{person}{Edward~H Kennedy}.} \bibinfo{year}{2019}\natexlab{}.
\newblock \showarticletitle{Nonparametric causal effects based on incremental propensity score interventions}.
\newblock \bibinfo{journal}{\emph{J. Amer. Statist. Assoc.}} \bibinfo{volume}{114}, \bibinfo{number}{526} (\bibinfo{year}{2019}), \bibinfo{pages}{645--656}.
\newblock


\bibitem[Lecu{\'e} and Mendelson(2017)]%
        {lecue2017sparse}
\bibfield{author}{\bibinfo{person}{Guillaume Lecu{\'e}} {and} \bibinfo{person}{Shahar Mendelson}.} \bibinfo{year}{2017}\natexlab{}.
\newblock \showarticletitle{Sparse recovery under weak moment assumptions}.
\newblock \bibinfo{journal}{\emph{Journal of the European Mathematical Society}} \bibinfo{volume}{19}, \bibinfo{number}{3} (\bibinfo{year}{2017}), \bibinfo{pages}{881--904}.
\newblock


\bibitem[Liu et~al\mbox{.}(2023)]%
        {liu2023reimagining}
\bibfield{author}{\bibinfo{person}{Lydia~T Liu}, \bibinfo{person}{Serena Wang}, \bibinfo{person}{Tolani Britton}, {and} \bibinfo{person}{Rediet Abebe}.} \bibinfo{year}{2023}\natexlab{}.
\newblock \showarticletitle{Reimagining the machine learning life cycle to improve educational outcomes of students}.
\newblock \bibinfo{journal}{\emph{Proceedings of the National Academy of Sciences}} \bibinfo{volume}{120}, \bibinfo{number}{9} (\bibinfo{year}{2023}), \bibinfo{pages}{e2204781120}.
\newblock


\bibitem[Narita(2021)]%
        {narita2021incorporating}
\bibfield{author}{\bibinfo{person}{Yusuke Narita}.} \bibinfo{year}{2021}\natexlab{}.
\newblock \showarticletitle{Incorporating ethics and welfare into randomized experiments}.
\newblock \bibinfo{journal}{\emph{Proceedings of the National Academy of Sciences}} \bibinfo{volume}{118}, \bibinfo{number}{1} (\bibinfo{year}{2021}), \bibinfo{pages}{e2008740118}.
\newblock


\bibitem[Neyman(1992)]%
        {neyman1992two}
\bibfield{author}{\bibinfo{person}{Jerzy Neyman}.} \bibinfo{year}{1992}\natexlab{}.
\newblock \showarticletitle{On the two different aspects of the representative method: the method of stratified sampling and the method of purposive selection}.
\newblock In \bibinfo{booktitle}{\emph{Breakthroughs in Statistics: Methodology and Distribution}}. \bibinfo{publisher}{Springer}, \bibinfo{pages}{123--150}.
\newblock


\bibitem[Owen and Varian(2020)]%
        {owen2020optimizing}
\bibfield{author}{\bibinfo{person}{Art~B Owen} {and} \bibinfo{person}{Hal Varian}.} \bibinfo{year}{2020}\natexlab{}.
\newblock \showarticletitle{Optimizing the tie-breaker regression discontinuity design}.
\newblock \bibinfo{journal}{\emph{Electronic Journal of Statistics}}  \bibinfo{volume}{14} (\bibinfo{year}{2020}), \bibinfo{pages}{4004--4027}.
\newblock


\bibitem[Pan et~al\mbox{.}(2017)]%
        {pan2017machine}
\bibfield{author}{\bibinfo{person}{Ian Pan}, \bibinfo{person}{Laura~B Nolan}, \bibinfo{person}{Rashida~R Brown}, \bibinfo{person}{Romana Khan}, \bibinfo{person}{Paul van~der Boor}, \bibinfo{person}{Daniel~G Harris}, {and} \bibinfo{person}{Rayid Ghani}.} \bibinfo{year}{2017}\natexlab{}.
\newblock \showarticletitle{Machine learning for social services: a study of prenatal case management in Illinois}.
\newblock \bibinfo{journal}{\emph{American journal of public health}} \bibinfo{volume}{107}, \bibinfo{number}{6} (\bibinfo{year}{2017}), \bibinfo{pages}{938--944}.
\newblock


\bibitem[Perchet et~al\mbox{.}(2016)]%
        {perchet2016batched}
\bibfield{author}{\bibinfo{person}{Vianney Perchet}, \bibinfo{person}{Philippe Rigollet}, \bibinfo{person}{Sylvain Chassang}, {and} \bibinfo{person}{Erik Snowberg}.} \bibinfo{year}{2016}\natexlab{}.
\newblock \showarticletitle{Batched bandit problems}.
\newblock \bibinfo{journal}{\emph{Annals of Statistics}} (\bibinfo{year}{2016}).
\newblock


\bibitem[Perdomo(2024)]%
        {perdomo2024relative}
\bibfield{author}{\bibinfo{person}{Juan~Carlos Perdomo}.} \bibinfo{year}{2024}\natexlab{}.
\newblock \showarticletitle{The Relative Value of Prediction in Algorithmic Decision Making}. In \bibinfo{booktitle}{\emph{International Conference on Machine Learning}}.
\newblock


\bibitem[Petry et~al\mbox{.}(2021)]%
        {petry2021associations}
\bibfield{author}{\bibinfo{person}{Laura Petry}, \bibinfo{person}{Chyna Hill}, \bibinfo{person}{Phebe Vayanos}, \bibinfo{person}{Eric Rice}, \bibinfo{person}{Hsun-Ta Hsu}, {and} \bibinfo{person}{Matthew Morton}.} \bibinfo{year}{2021}\natexlab{}.
\newblock \showarticletitle{Associations between the Vulnerability Index-Service Prioritization Decision Assistance Tool and returns to homelessness among single adults in the United States}.
\newblock \bibinfo{journal}{\emph{Cityscape}} \bibinfo{volume}{23}, \bibinfo{number}{2} (\bibinfo{year}{2021}), \bibinfo{pages}{293--324}.
\newblock


\bibitem[Rolf et~al\mbox{.}(2020)]%
        {rolf2020balancing}
\bibfield{author}{\bibinfo{person}{Esther Rolf}, \bibinfo{person}{Max Simchowitz}, \bibinfo{person}{Sarah Dean}, \bibinfo{person}{Lydia~T Liu}, \bibinfo{person}{Daniel Bjorkegren}, \bibinfo{person}{Moritz Hardt}, {and} \bibinfo{person}{Joshua Blumenstock}.} \bibinfo{year}{2020}\natexlab{}.
\newblock \showarticletitle{Balancing competing objectives with noisy data: Score-based classifiers for welfare-aware machine learning}. In \bibinfo{booktitle}{\emph{International Conference on Machine Learning}}. PMLR, \bibinfo{pages}{8158--8168}.
\newblock


\bibitem[Schochet(2009)]%
        {schochet2009statistical}
\bibfield{author}{\bibinfo{person}{Peter~Z Schochet}.} \bibinfo{year}{2009}\natexlab{}.
\newblock \showarticletitle{Statistical power for regression discontinuity designs in education evaluations}.
\newblock \bibinfo{journal}{\emph{Journal of Educational and Behavioral Statistics}} \bibinfo{volume}{34}, \bibinfo{number}{2} (\bibinfo{year}{2009}), \bibinfo{pages}{238--266}.
\newblock


\bibitem[Shinn and Richard(2022)]%
        {shinn2022allocating}
\bibfield{author}{\bibinfo{person}{Marybeth Shinn} {and} \bibinfo{person}{Molly~K Richard}.} \bibinfo{year}{2022}\natexlab{}.
\newblock \bibinfo{title}{Allocating homeless services after the withdrawal of the vulnerability index--service prioritization decision assistance tool}.
\newblock , \bibinfo{numpages}{378--382}~pages.
\newblock


\bibitem[Swaminathan and Joachims(2015)]%
        {swaminathan2015batch}
\bibfield{author}{\bibinfo{person}{Adith Swaminathan} {and} \bibinfo{person}{Thorsten Joachims}.} \bibinfo{year}{2015}\natexlab{}.
\newblock \showarticletitle{Batch learning from logged bandit feedback through counterfactual risk minimization}.
\newblock \bibinfo{journal}{\emph{The Journal of Machine Learning Research}} \bibinfo{volume}{16}, \bibinfo{number}{1} (\bibinfo{year}{2015}), \bibinfo{pages}{1731--1755}.
\newblock


\bibitem[Toros and Flaming(2017)]%
        {toros2017prioritizing}
\bibfield{author}{\bibinfo{person}{Halil Toros} {and} \bibinfo{person}{Daniel Flaming}.} \bibinfo{year}{2017}\natexlab{}.
\newblock \showarticletitle{Prioritizing which homeless people get housing using predictive algorithms}.
\newblock \bibinfo{journal}{\emph{Available at SSRN 2960410}} (\bibinfo{year}{2017}).
\newblock


\bibitem[Vaithianathan and Kithulgoda(2020)]%
        {vaithianathan2020using}
\bibfield{author}{\bibinfo{person}{Rhema Vaithianathan} {and} \bibinfo{person}{Chamari~I Kithulgoda}.} \bibinfo{year}{2020}\natexlab{}.
\newblock \showarticletitle{Using predictive risk modeling to prioritize services for people experiencing homelessness in Allegheny County: Methodology update}.
\newblock  (\bibinfo{year}{2020}).
\newblock


\bibitem[Van Der~Vaart et~al\mbox{.}(1996)]%
        {van1996weak}
\bibfield{author}{\bibinfo{person}{Aad~W Van Der~Vaart}, \bibinfo{person}{Jon~A Wellner}, \bibinfo{person}{Aad~W van~der Vaart}, {and} \bibinfo{person}{Jon~A Wellner}.} \bibinfo{year}{1996}\natexlab{}.
\newblock \bibinfo{booktitle}{\emph{Weak convergence}}.
\newblock \bibinfo{publisher}{Springer}.
\newblock


\bibitem[Wainwright(2019)]%
        {wainwright2019high}
\bibfield{author}{\bibinfo{person}{Martin~J Wainwright}.} \bibinfo{year}{2019}\natexlab{}.
\newblock \bibinfo{booktitle}{\emph{High-dimensional statistics: A non-asymptotic viewpoint}}. Vol.~\bibinfo{volume}{48}.
\newblock \bibinfo{publisher}{Cambridge university press}.
\newblock


\bibitem[Yaskov(2016)]%
        {yaskov2016controlling}
\bibfield{author}{\bibinfo{person}{Pavel Yaskov}.} \bibinfo{year}{2016}\natexlab{}.
\newblock \showarticletitle{Controlling the least eigenvalue of a random Gram matrix}.
\newblock \bibinfo{journal}{\emph{Linear Algebra Appl.}}  \bibinfo{volume}{504} (\bibinfo{year}{2016}), \bibinfo{pages}{108--123}.
\newblock


\bibitem[Zhao(2023)]%
        {zhao2023adaptive}
\bibfield{author}{\bibinfo{person}{Jinglong Zhao}.} \bibinfo{year}{2023}\natexlab{}.
\newblock \showarticletitle{Adaptive Neyman Allocation}.
\newblock  (\bibinfo{year}{2023}).
\newblock


\end{thebibliography}

% Appendix
\appendix

\section{Proofs}

\begin{proof}[Proof of Proposition \ref{prop:equal-minmax}]
    Fix an policy $p \in \mathcal{P}$. Since the efficiency bound is monotone in $v_0(X)$ and $v_1(X)$ for any $X$, we have that 
    \begin{align*}
        \max_{v_0, v_1 \in \Sigma_{\infty, C}} \E\left[\frac{v_1(X)}{p(X)} + \frac{v_0(X)}{1- p(X)}\right] = \E\left[\frac{C}{p(X)} + \frac{C}{1- p(X)}\right].
    \end{align*}
    That is, the functions $v_0, v_1$ achieving the max are $v_0(X) = v_1(X) = C$ regardless of $p$. Since $C > 0$ is constant with respect to $X$, it does not affect the maximization. I.e., if $p$ is an optimal solution to the simplified problem that satisfies
    \begin{align*}
         \E\left[\frac{1}{p(X)} + \frac{1}{1- p(X)}\right] \leq \E\left[\frac{1}{q(X)} + \frac{1}{1- q(X)}\right] \,\,\, \forall q \in \mathcal{P}
    \end{align*}
    then it must hold that 
    \begin{align*}
         \E\left[\frac{C}{p(X)} + \frac{C}{1- p(X)}\right] \leq \E\left[\frac{C}{q(X)} + \frac{C}{1- q(X)}\right] \,\,\, \forall q \in \mathcal{P}
    \end{align*}
    and so $p$ is optimal for the original minmax problem as well. 
\end{proof}

\begin{proof}[Proof of Proposition \ref{prop:estimated-minmax}]
    The proof follows by similar logic as Proposition 1: since the efficiency bound is monotone in $v_0(X)$ and $v_1(X)$ for any $X$, the max is achieved when both functions take their largest possible value, i.e.,  
    \begin{align*}
        \max_{v_0, v_1 \in \Sigma_{a}} \E\left[\frac{v_1(X)}{p(X)} + \frac{v_0(X)}{1- p(X)}\right] = \E\left[\frac{a(X)\hat{v}_0(X)}{p(X)} + \frac{\hat{v}_0(X)}{1- p(X)}\right].
    \end{align*}
    and so maximizing the later function is equivalent to maximizing the former.
\end{proof}

Next, we turn to proving the main sample complexity results. We first introduce some additional notation. Define $f(p(X)) = \frac{C(X)}{p(X)} + \frac{1}{1- p(X)}$. 
Define $p(X, \lambda)$ to be the optimal $p$ with fixed dual parameters $\lambda$, i.e., the solution to Equation \ref{eq-optimal-p}. Define $h_j(\lambda) = \E[g(p(X, \lambda), X)]$ and $h_j^n(\lambda) = \frac{1}{n}\sum_{i = 1}^n g(p(X_i, \lambda), X_i)$. 

\begin{lemma}
    $\frac{1}{n}\sum_{i = 1}^n f(p(X_i))$ is 32$a_{\text{min}}$-strongly convex in $p(X_1)...p(X_n)$. \label{lemma:strict-convex}
\end{lemma} 
\begin{proof}
    Taking second derivatives, we have that 
    \begin{align*}
        \frac{\partial^2 f}{\partial p(X_i)^2}  = \frac{2a_0(X)}{(1 - p(X_i))^3 } + \frac{2a_1(X)}{p(X_i)^3 } \geq 2a_{\text{min}}\left(\frac{1}{(1 - p(X_i))^3 } + \frac{1}{p(X_i)^3 }\right) \geq 32a_{\text{min}}
    \end{align*}
    where the second inequality follows because the expression in parentheses is minimized at $p(X) = \frac{1}{2}$. Since the summation is separable, its Hessian is a diagonal matrix with entries lower bounded by 32$a_{\text{min}}$ for $a_{\text{min}} > 0$, implying that $f$ must be strongly convex. 
\end{proof}

In what follows below, we let $L = 32 a_{\text{min}}$ denote the strong convexity constant of $f$.

\begin{lemma}
    For all $j$, $h_j$ and $h_j^n$ are both $\frac{1}{L}-$Lipschitz in the $\ell_1$ norm.   \label{lemma:constraints-smooth}
\end{lemma}
\begin{proof}
     Within any range where the optimal $p(X, \lambda)$ is on the boundary of the feasible set, $p(X, \lambda)$ does not depend on the dual variables. If the optimal solution is not on the boundary, it is characterized equivalently by the solution to the first-order condition 
    \begin{align*}
        m(p(X), \lambda) \triangleq -\frac{C(X)}{p(X)^2} + \frac{1}{(1 - p(X))^2} + \sum_{j = 1}^J \nabla g_j(p(X), X) \lambda_j  = 0.
    \end{align*}
    % The first-order condition has a solution in the interval $[\gamma, 1 - \gamma]$ whenever 
    % \begin{align*}
    %    \frac{1}{\nabla g_j(p(X), \lambda)} \left(-\frac{1}{\gamma^2} + \frac{1}{(1 - \gamma)^2} - \sum_{k \neq k} \nabla g_k(p(X), X) \lambda_k \right)
    % \end{align*}

    We will bound the derivative $\frac{\partial p(X, \lambda)}{\partial \lambda_j}$ for the set of $\lambda$ where $p(X, \lambda) \in (\gamma, 1-\gamma)$.  When $p(X, \lambda) \in \{\gamma, 1-\gamma\}$, $\frac{\partial p(X, \lambda)}{\partial \lambda_j}$ is either undefined (on the measure zero set of boundary points where $p(X, \lambda)$ is not differentiable) or else $\frac{\partial p(X, \lambda)}{\partial \lambda_j} = 0$ (over regions where  $p(X, \lambda)$ is constant at either $\gamma$ or $1 - \gamma$ and hence does not depend on $\lambda$). Accordingly, to bound the Lipschitz constant, it suffices to bound $\frac{\partial p(X, \lambda)}{\partial \lambda_j}$ just over the region where $p(X, \lambda)$ is given by the solution to the first order condition. Within this region,  applying the implicit function theorem to the first-order condition yields that 
    \begin{align*}
        \frac{\partial p(X, \lambda)}{\partial \lambda_j} = -\frac{\partial m(p(X), \lambda)}{\partial p(X)}^{-1} \frac{\partial m(p(X), \lambda)}{\partial \lambda_j}
    \end{align*}
    where we can bound $\frac{\partial m}{\partial p(X)} \geq L$ (via Lemma \ref{lemma:strict-convex}) and $\left\lvert \frac{\partial m}{\partial \lambda_j} \right\rvert = |\nabla g_j(p(X), X)| \leq 1$ by assumption that $g_j$ is Lipschitz. Combining, we obtain 
    \begin{align*}
        \left\lvert \frac{\partial p(X, \lambda)}{\partial \lambda_j} \right\rvert \leq \frac{1}{L} \quad \forall j = 1...J.
    \end{align*}
    Using again the assumption that the $g_j$ are 1-Lipschitz in $p$, we have that 
    \begin{align*}
       \left\lvert \frac{\partial g_j(p(X, \lambda))}{\partial \lambda_j} \right\rvert \leq  \left\lvert \frac{\partial g_j(p(X, \lambda))}{\partial p(X, \lambda)} \right\rvert \left\lvert \frac{\partial p(X, \lambda)}{\partial \lambda_j} \right\rvert \leq \frac{1}{L}
    \end{align*}
    and so we have proved that for each $g_j$, $||\nabla_\lambda g_j(p(X, \lambda))||_\infty \leq \frac{1}{L}$. Taking expectations over $\PP$ and the empirical distribution respectively yields that $||\nabla_\lambda h_j(\lambda))||_\infty \leq \frac{1}{L}$ and $||\nabla_\lambda h_j^n(\lambda))||_\infty \leq \frac{1}{L}$. Finally, the conclusion in the lemma follows because the $\ell_\infty$ norm is dual to the $\ell_1$.   
\end{proof}
\begin{proof}[Proof of Proposition \ref{prop-sample-complexity}]
We start by proving that the expectations in the constraints are well-approximated with high probability. The main idea is to discretize the range of the dual variables, since by assumption $\widehat{\lambda} \in [0, d]^J$ with probability at least $1 - \delta_1$. Let $N([0, d]^J, \epsilon, ||\cdot||_1)$ denote the $\ell_1$ covering number of the set $[0, d]^J$. Let $B_1$ be the $\ell_1$ unit ball in $R^J$. Standard bounds on covering numbers (e.g., \citet{wainwright2019high}, Lemma 5.7) imply that 
$N([0, d]^J, \epsilon, ||\cdot||_1) \leq \left(1 + \frac{2}{\epsilon}\right)^J \frac{\text{vol}([0, d]^J)}{\text{vol}(B_1)}$. Since $\text{vol}(B_1) = \frac{2^J}{J!}$ and $\text{vol}([0,d]^J) = d^J$, we obtain that 
\begin{align*}
    \log N([0, d]^J, \epsilon, ||\cdot||_1) \leq J \log \left(1+ \frac{2}{\epsilon}\right) + \log J! + J \log \frac{d}{2} = \Theta\left(J\left(\log \frac{1}{\epsilon} + \log J + \log d\right)\right). 
\end{align*}
We fix a $\epsilon$-covering $\Lambda$ of at most this size. We have that each constraint function $h_j$ is $\frac{1}{2}-$Lipschitz in the $\ell_1$ norm. Accordingly, we condition on the event that 
\begin{align}
    |h_j^n(\lambda) - h_j(\lambda)| \leq \epsilon \quad \forall \lambda \in \Lambda, j = 1...J \label{concentration-grid}
\end{align}
which by a standard Hoeffding bound occurs with probability at least $1 - 2\exp(-2n\epsilon^2)$ for each $(\lambda, J)$ individually, and hence with probability at least $1 - 2J |\Lambda|\exp(-2n\epsilon^2)$ in total via union bound. We also condition on the event that $||\widehat{\lambda}||_\infty \leq d$, which occurs with probability at least $1 - \delta_1$ by assumption. Since $\widehat{\lambda} \geq 0$ by dual feasibility, there must be a point $\lambda' \in \Lambda$ with $||\widehat{\lambda} - \lambda'||_1 \leq \epsilon$.  We have via triangle inequality that 
\begin{align*}
    |h^n_j(\widehat{\lambda}) -  h_j(\widehat{\lambda})| \leq  |h^n_j(\widehat{\lambda}) -  h^n_j(\lambda')| + | h^n_j(\lambda') - h_j(\lambda')| + |h_j(\lambda') - h_j(\widehat{\lambda})|.
\end{align*}
Since $||\widehat{\lambda} - \lambda'||_1 \leq \epsilon$, Lemma \ref{lemma:constraints-smooth} implies that both the first and third terms are at most $\epsilon L$. Event \ref{concentration-grid} implies that the third term is at most $\epsilon$. Summing up, we have that 
\begin{align}
    |h^n_j(\widehat{\lambda}) -  h_j(\widehat{\lambda})| \leq (2L + 1) \epsilon \quad \forall j = 1...J \label{eq-near-feasible}
\end{align}
with probability at least $1 - \delta_1 - 2J |\Lambda|\exp(-2n\epsilon^2)$. Conditional on this event, 
\begin{align*}
    h_j(\widehat{\lambda}) \leq h^n_j(\widehat{\lambda}) + (2L + 1)\epsilon \leq c_j + (2L + 1)\epsilon
\end{align*}
where the last line uses primal feasibility for the sample problem. 

Next, we turn to showing that the objective function value is near-optimal. Essentially, the challenge is that if the sample dual variables cause errors in the other direction, making the constraints overly conservative, this may result in a worse value of the objective function. Writing out the population dual objective function and using optimality of $\lambda^*$ for the maximization, we have that 
    \begin{align*}
        \left(\E[f(\widehat{p}(X)] + \sum_j (\E[g_j(\widehat{p}(X), X)] - c_j) \widehat{\lambda}_j \right)- \left(\E[f(p^*(X)] + \sum_j (\E[g_j(p^*(X), X)] - c_j) \lambda^*_j \right) \leq 0.
    \end{align*}
    Complementary slackness for the population problem implies that $\sum_j (\E[g_j(p^*(X), X)] - c_j) \lambda^*_j = 0$, so we have
    \begin{align*}
        \E[f(\widehat{p}(X)]   - \E[f(p^*(X)]  \leq \sum_j -(\E[g_j(\widehat{p}(X), X)] - c_j) \widehat{\lambda}_j.
    \end{align*}
    Now, for each constraint $j$, we consider two cases. Intuitively, these divide whether there is a large or small amount of slack in the $j$th constraint. If there is a small amount of slack, this cannot hurt the objective value too much because the dual variables (and hence the ``shadow price" for the constraint) are bounded. If there is a large amount slack, then this implies that there is also slack in the sample problem and so leaving slack must not have hurt the objective value. More formally, the first case is that $\E[g_j(\widehat{p}(X), X)] - c_j \geq - (2L+1)\epsilon$. Since $\widehat{\lambda}_j \geq 0$ this implies that 
    \begin{align*}
        -(\E[g_j(\widehat{p}(X), X)] - c_j) \widehat{\lambda}_j \leq (2L+1)\epsilon \cdot \widehat{\lambda}_j. 
    \end{align*}
    The second case is that $\E[g_j(\widehat{p}(X), X)] - c_j < -(2L+1)\epsilon$. Using Equation \ref{eq-near-feasible}, we have that 
    \begin{align*}
        \frac{1}{n}\sum_{i = 1}^n g_j(\widehat{p}(X_i), X_i) - c_j = h^n_j(\widehat{\lambda}) - c_j < h_j(\widehat{\lambda}) - c_j + (2L+1)\epsilon < 0.
    \end{align*}
  In this case, complementary slackness for the sample problem implies that $\widehat{\lambda}_j = 0$ and so $(\E[g_j(\widehat{p}(X), X)] - c_j) \widehat{\lambda}_j = 0$. Combining these cases and summing over the two constraints, we conclude that 
    \begin{align*}
        \E[f(\widehat{p}(X)]   - \E[f(p^*(X)]  \leq (2L+1)J \epsilon \cdot ||\widehat{\lambda}_j||_\infty \leq  (2L+1)J d\epsilon.
    \end{align*}
Rescaling $\epsilon$ and setting the failure probability for event \ref{concentration-grid} in order to obtain the guarantees in the theorem statement, we require 
\begin{align*}
    n = O\left(\frac{J^2d^2L^2}{\epsilon^2}\left(\log\frac{J}{\delta_2} + J \log \frac{Jd}{\epsilon}\right)\right) = O\left(\frac{J^3d^2}{\epsilon^2a_{\text{min}}^2}\log\frac{Jd}{\epsilon\delta_2} \right).
\end{align*}
\end{proof}

\begin{proof}[Proof of Proposition \ref{corollary-small-ball}]
Strong convexity of the objective function implies that there are a unique set of sample optimal dual parameters $\widehat{\lambda}$. The duals must solve the stationarity condition 
\begin{align*}
    \sum_{j = 1}^J \lambda_j \nabla_{p(X_i)} \left(\frac{1}{n}\sum_{k = 1}^n g_j(p(X_k), X_k)\right) = \nabla_{p(X_i)} \left(\frac{1}{n}\sum_{k = 1}^n f(p(X_k))\right) \quad \forall i = 1...n
\end{align*}
which, using separability of the objective and constraints, reduces to 
\begin{align*}
    \sum_{j = 1}^J \lambda_j \nabla_{p(X_i)}  g_j(p(X_i), X_i) = \nabla_{p(X_i)} f(p(X_i)) \quad \forall i = 1...n.
\end{align*}
Let $G$ denote the matrix 
\begin{align*}
   G =  \begin{bmatrix}
           g(X_1) \\
           g(X_2) \\
           \vdots \\
           g(X_n)
    \end{bmatrix} 
\end{align*}
where $g(X_i)$ was defined earlier to be the vector of coefficients for the linear constraints $[g_1(X_i)...g_J(X_i)]$. Note that each row is a iid random variable since we assume that the $X_i$ are iid. Finally, let $F$ denote the right-hand side vector 
\begin{align*}
   F =  \begin{bmatrix}
           \nabla_{p(X_1)} f(p(X_1)) \\
           \nabla_{p(X_2)} f(p(X_2)) \\
           \vdots \\
           \nabla_{p(X_n)} f(p(X_n))
    \end{bmatrix} 
\end{align*}
We can now characterize the duals as the solution of the linear system
\begin{align*}
    G \widehat{\lambda} = F.
\end{align*}
We start by noting that under the small-ball condition 
\begin{align*}
    \inf_{||v||_2 \leq 1}\Pr\left(|v^T \nabla g(X_i)| \geq \alpha\right) \geq \beta
\end{align*}
we have via \citet{lecue2017sparse} (using the statement of the result in \citep{yaskov2016controlling}) 
\begin{align*}
    \Pr\left(\sigma_{\text{min}}(n^{-1} G^T G) \geq \frac{\alpha^2 \beta}{2}\right) \geq 1 - \Theta(\exp(-\beta^2 n)).
\end{align*}
where $\sigma_{\text{min}}(\cdot)$ denotes the smallest singular value of a matrix. We condition on this event in what follows. We apply the singular value decomposition $G = U \Sigma V^T$. Let $\sigma_1...\sigma_p$ be the singular values of $G$. The singular values of $G^T G$, which are controlled under the small ball condition, are $\sigma_1^2...\sigma_p^2$. We have that 
\begin{align*}
    \min_{i = 1...p} \sigma_i^2 \geq n \frac{\alpha^2 \beta}{2}.
\end{align*}
Since all of the singular values are bounded away from 0, $G$ has full column rank and we can write the solution to the linear system defining $\widehat{\lambda}$ as 
\begin{align*}
    \widehat{\lambda} = (G^TG)^{-1}G^TF.
\end{align*}
 Plugging in the SVD gives the usual form for the pseudoinverse 
\begin{align*}
    (G^TG)^{-1}G^T = V \Sigma^+U^T
\end{align*}
where $\Sigma^+$ is the pseudoinverse of $\Sigma$. Accordingly, the singular values of $(G^TG)^{-1}G^T$ are simply $\sigma_1^{-1}...\sigma_p^{-1}$. Since we have that 
\begin{align*}
    \max_{i = 1...p} \sigma^{-1}_i \leq \frac{\sqrt{2}}{\alpha \sqrt{\beta n}}.
\end{align*}
this in turn bounds the operator norm of $(G^TG)^{-1}G^T$ so that
\begin{align*}
    ||\widehat{\lambda}||_2 &\leq ||(G^TG)^{-1}G^T||_2 ||F||_2\\
    &\leq \frac{\sqrt{2}}{\alpha \sqrt{\beta n}} ||F||_2 \\
    &\leq \frac{\sqrt{2}}{\alpha \sqrt{\beta}} ||F||_\infty.
\end{align*}
Finally, a simple calculation shows that $||F||_\infty \leq \max_{p(X) \in [\gamma, 1-\gamma]} |\nabla f(p(X))| \leq \frac{a_{\text{max}}}{\gamma^2}$. Putting it all together, we have that 
\begin{align*}
    ||\widehat{\lambda}||_\infty &\leq ||\widehat{\lambda}||_2 \leq \frac{\sqrt{2}a_{\text{max}}}{\alpha \sqrt{\beta} \gamma^2}
\end{align*}
and the result follows by substituting this bound into Proposition \ref{prop-sample-complexity}.
\end{proof}

\begin{proof}[Proof of Proposition \ref{prop:reverse-markov}]
    Since we have by assumption that $u(X) \in [0,1]$, the reverse Markov inequality for bounded random variables, i.e., applying the Markov inequality to the nonnegative random variable $1 - u(X)$, gives that 
    \begin{align*}
        \Pr\left(u(X) \geq \alpha\right) \geq \frac{\E[u(X)] - \alpha}{1- \alpha}
    \end{align*}
    for any $\alpha < \E[u(X)]$. In the setting of Problem \ref{problem:main}, $g(X) = [u(X), 1]$ and so since $u(X) \in [0,1]$, the minimizing value of $\nu$ in the small condition is always $\nu = [1, 0]$ and $g(X)^T\nu = u(X)$. Therefore, we have that
    \begin{align*}
        \inf_{||\nu||_2 = 1} \Pr\left(|g(X)^T\nu| \geq \alpha\right) = \Pr\left(u(X) \geq \alpha\right)
    \end{align*}
    and so the small ball condition reduces to verifying the tail of $u(X)$ as analyzed above. 
\end{proof}

\section{Additional experimental results}

\paragraph{Sensitivity analysis to varrying treatment effects}

In Figure \ref{fig:sensitivity-beta} we present versions of our main empirical results that vary the hypothetical treatment effects. We fix the budget $b = 0.3$ and present results for $\beta = 0.05$ and $\beta = 0.15$ (compared to our main results for $\beta = 0.1$). The main conclusions are unchanged: the tradeoff curve between utility and sample size is highly concave, with most of the optimal utility being obtainable with relatively little sample size inflation. As expected, the absolute number of samples required for all designs (including the ideal RCT) increase when the treatment effect is small, but the relative differences between the alternate assignment policies are unchanged.

\paragraph{Sensitivity to treatment effect heterogeneity} Finally, we test the extent to which our results vary depending on the structure of heterogeneity present in the simulated treatment effects. We emphasize that our policy optimization process does not use any information about treatment effects and will permit unbiased estimation of such effects regardless of what the true structure is. The pattern of heterogeneity only potentially impacts the efficiency of estimation for data collected under the policy. The primary model used in our analysis so far has treatment effects increasing with $\Pr(Y(0) =1|X)$, which is typically the hope of service providers (people with higher need benefit more). Here, we simulate two other possibilities to test sensitivity to this model. First, where treatment effects have a "U-shaped" curve, with highest treatment effects for those with middling values of $\Pr(Y(0) = 1|X)$. We construct this by setting $\tau$ as the unique quadratic function of $\Pr(Y(0) = 1|X)$ which takes value 0 at $\Pr(Y(0) = 1|X) = 0$ or $\Pr(Y(0) = 1|X) = 1$, peaks at $\Pr(Y(0) = 1|X) = 0.5$, and produces an average treatment effect of $0.1$. Second, where treatment effects are (near) constant, produced by setting $\tau(X) = \min\{0.1, \Pr(Y(0) = 1|X)\}$ (i.e., a constant treatment effect of 0.1, unless capped so that $\Pr(Y(1) = 1|X)$ is nonnegative). The results are shown in Figure \ref{fig:te-models}. We find that results are almost unchanged compared to before, indicating that these forms of potential heterogeneity in treatment effects do not significantly alter the utility-sample size Pareto frontier produced by our method.

\begin{figure}
    \centering
    \includegraphics[width=2.5in]{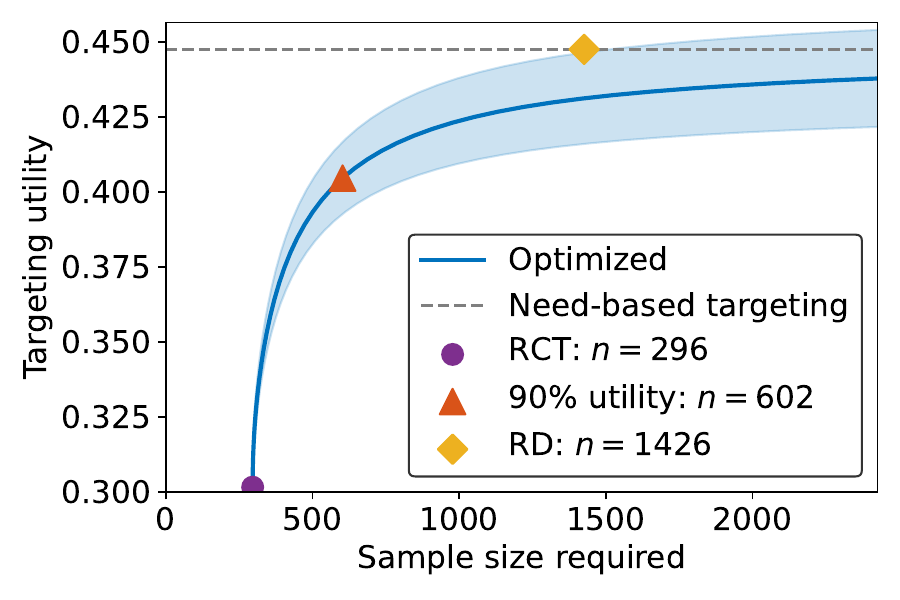}
        \includegraphics[width=2.5in]{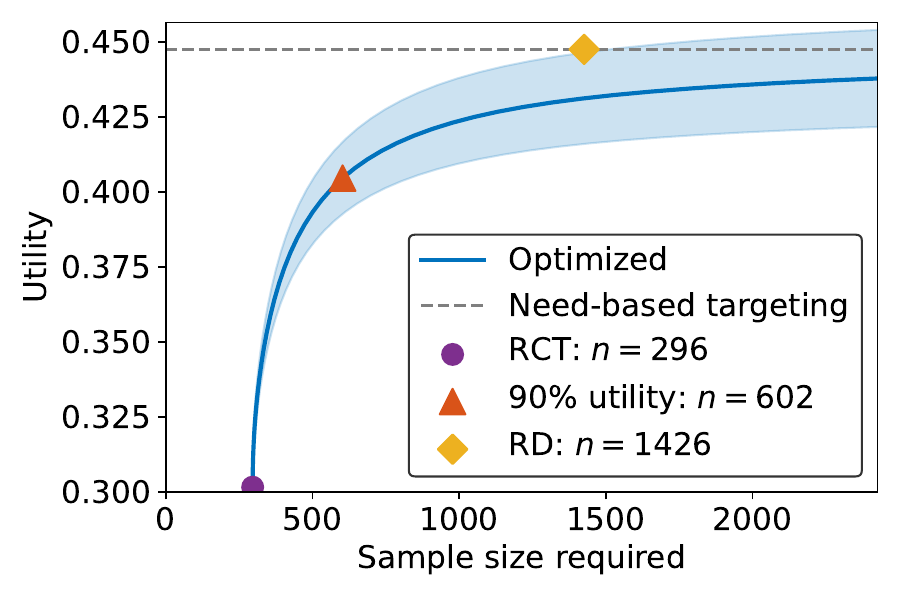}
        \includegraphics[width=2.5in]{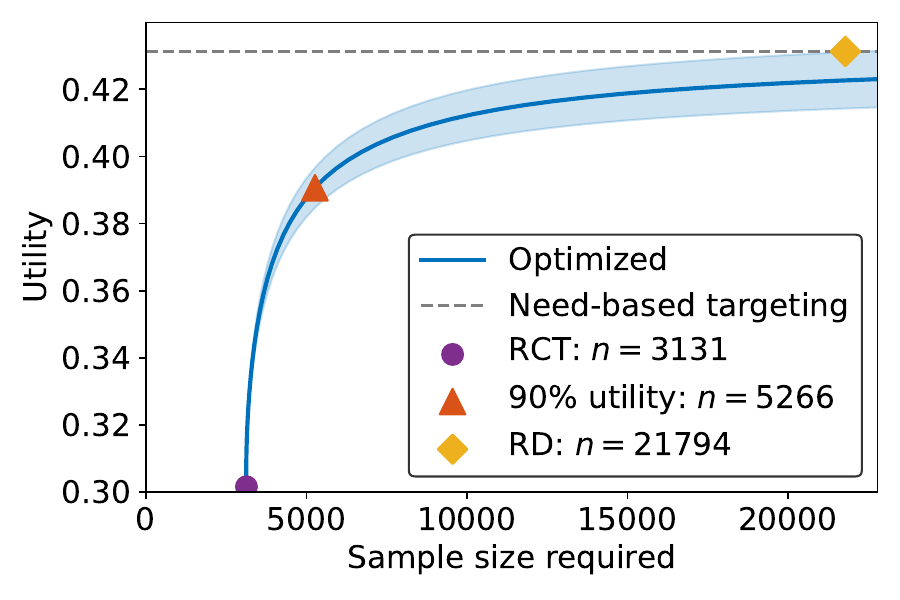}
        \includegraphics[width=2.5in]{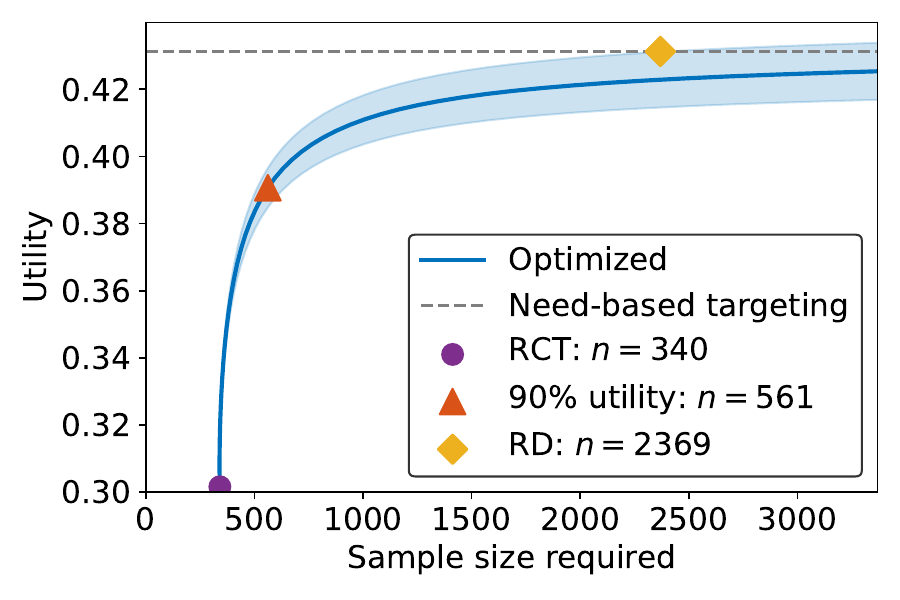}

    \caption{Utility-sample size tradeoff curves for $\beta = 0.05$ (left) and $\beta = 0.15$ (right). Top row: housing data. Bottom row: reentry data.}
    \label{fig:sensitivity-beta}
\end{figure}

\begin{figure}
    \centering
    \includegraphics[width=2.5in]{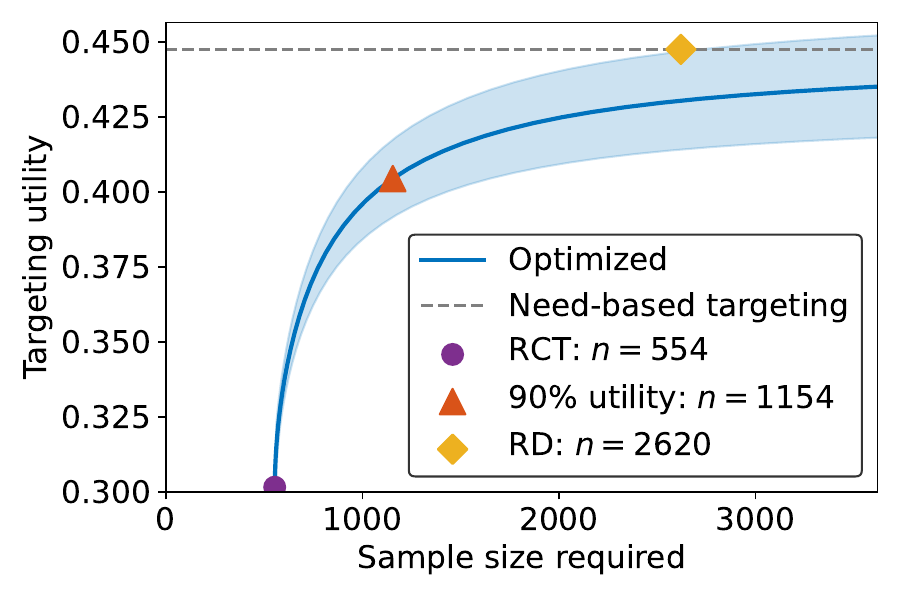}
    \includegraphics[width=2.5in]{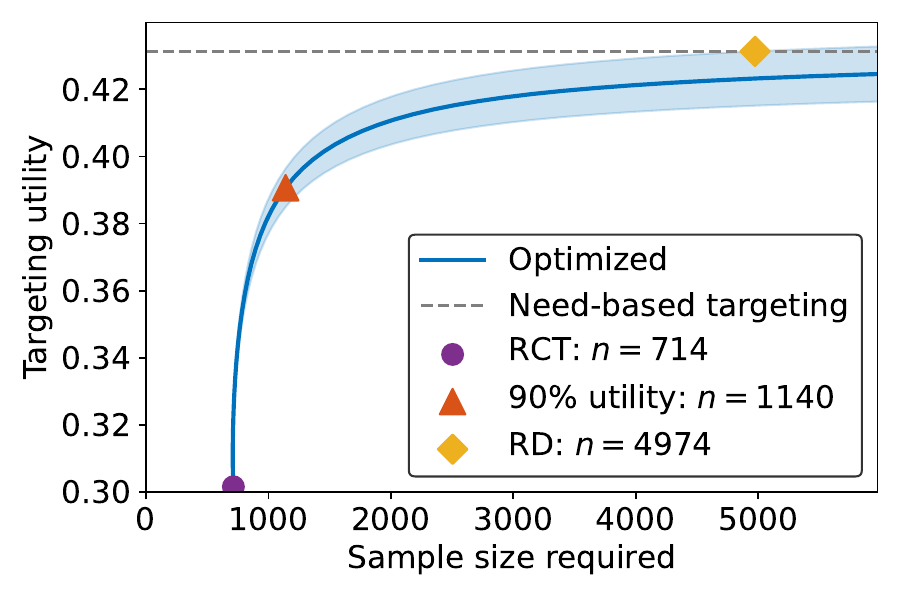}
    \includegraphics[width=2.5in]{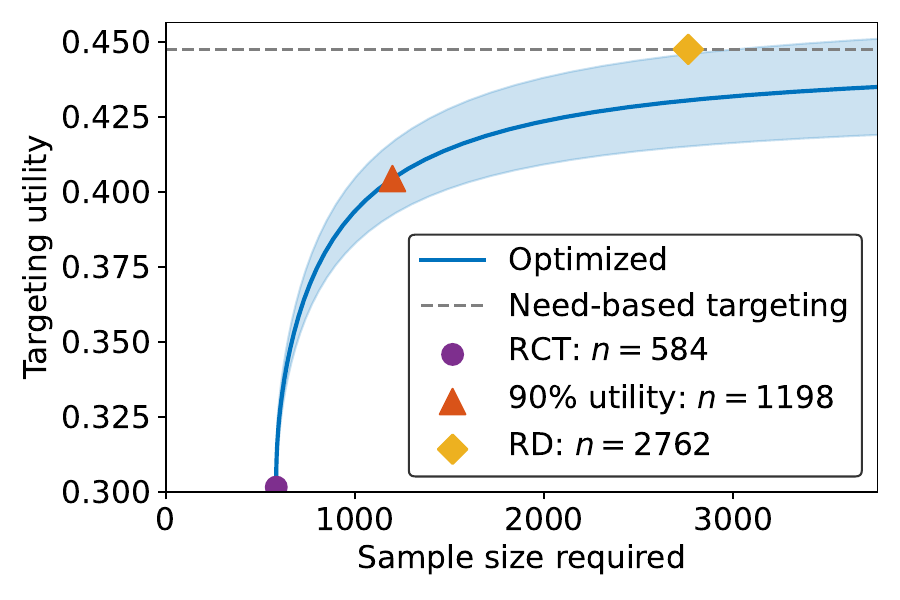}
    \includegraphics[width=2.5in]{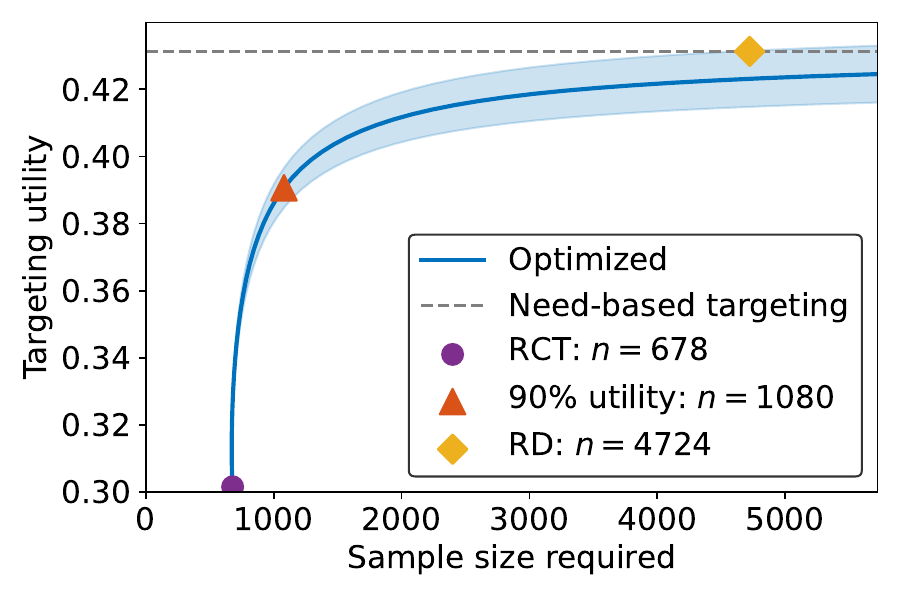}
    \caption{Utility-sample size tradeoffs under alternative treatment effect shapes. Top: U-shaped. Bottom: constant. Left: housing data. Right: reentry data.}
    \label{fig:te-models}
\end{figure}

\end{document}